\documentclass[nohyperref]{article}

\usepackage{microtype}
\usepackage{graphicx}
\usepackage{subfigure}
\usepackage{booktabs} % for professional tables

\usepackage{hyperref}

\usepackage[accepted]{icml2023}

\usepackage{amsmath}
\usepackage{amssymb}
\usepackage{mathtools}
\usepackage{amsthm}
\usepackage{soul}

\usepackage{diagbox}
\usepackage{makecell}
\usepackage{lipsum}
\usepackage{bbm}
\usepackage{tablefootnote}
\usepackage[capitalize,noabbrev]{cleveref}
\usepackage[normalem]{ulem}
\theoremstyle{plain}
\newtheorem{theorem}{Theorem}

\newtheorem{lemma}{Lemma}

\newtheorem{requirement}{Requirement}
\newtheorem{definition}{Definition}

\usepackage[textsize=tiny]{todonotes}

\def\ie{\emph{i.e.\/}}

\def\eg{\emph{e.g.\/}}
\def\wrt{\emph{w.r.t.\/}}
\def\ie{\emph{i.e.\/}}
\def\st{\emph{s.t.\/}}

\icmltitlerunning{HarsanyiNet: Computing Accurate Shapley Values in a Single Forward Propagation}

\begin{document}

\twocolumn[
\icmltitle{HarsanyiNet: Computing Accurate Shapley Values \\ in a Single Forward Propagation}



\icmlsetsymbol{equal}{*}

\begin{icmlauthorlist}
\icmlauthor{Lu Chen}{equal,yyy}
\icmlauthor{Siyu Lou}{equal,yyy,comp}
\icmlauthor{Keyan Zhang}{yyy}
\icmlauthor{Jin Huang}{yyy}
\icmlauthor{Quanshi Zhang}{yyy}
%\icmlauthor{Firstname6 Lastname6}{sch,yyy,comp}
%\icmlauthor{Firstname7 Lastname7}{comp}
%\icmlauthor{Firstname8 Lastname8}{sch}
%\icmlauthor{Firstname8 Lastname8}{yyy,comp}
\end{icmlauthorlist}

\icmlaffiliation{yyy}{Shanghai Jiao Tong University, China}
\icmlaffiliation{comp}{Eastern Institute for Advanced Study, China}
%\icmlaffiliation{sch}{School of ZZZ, Institute of WWW, Location, Country}

\icmlcorrespondingauthor{Quanshi Zhang is the corresponding author. He is with the Department of Computer Science and Engineering,
the John Hopcroft Center, at the Shanghai Jiao Tong University, China.}{zqs1022@sjtu.edu.cn}
%\icmlcorrespondingauthor{Firstname2 Lastname2}{first2.last2@www.uk}

\icmlkeywords{Machine Learning, ICML}

\vskip 0.3in
]



\printAffiliationsAndNotice{\icmlEqualContribution} % otherwise use the standard text.

\begin{abstract}
The Shapley value is widely regarded as a trustworthy attribution metric. However, when people use Shapley values to explain the attribution of input variables of a deep neural network (DNN), it usually requires a very high computational cost to approximate relatively accurate Shapley values in real-world applications. Therefore, we propose a novel network architecture, the \textit{HarsanyiNet}, which makes inferences on the input sample and simultaneously computes the exact Shapley values of the input variables in a single forward propagation. The HarsanyiNet is designed on the theoretical foundation that the Shapley value can be reformulated as the redistribution of Harsanyi interactions encoded by the network. 
\end{abstract}

\section{Introduction}
\label{Introduction}
Explainable artificial intelligence (XAI) has received considerable attention in recent years. A typical direction of explaining deep neural networks (DNNs) is to estimate the salience\slash importance\slash contribution of an input variable (\eg, a pixel of an image or a word in a sentence) to the network output. Related studies have been termed the \textit{attribution methods}~\citep{bach2015pixel, selvaraju2017grad, sundararajan2017axiomatic, lundberg2017unified}. In comparison with most attribution methods designed without solid theoretical supports, the Shapley value~\citep{shapley1953npersongame} has been proved the only solution in game theory that satisfies the \textit{linearity}, \textit{dummy}, \textit{symmetry}, and \textit{efficiency} axioms~\citep{young1985monotonic}. Therefore, the Shapley value is widely considered a relatively trustworthy attribution for each input variable.

% nullity 

\begin{figure}[t]
\begin{center}
\centerline{\includegraphics[width=0.95\linewidth,trim={0cm 2cm 12cm 0.3cm},clip]{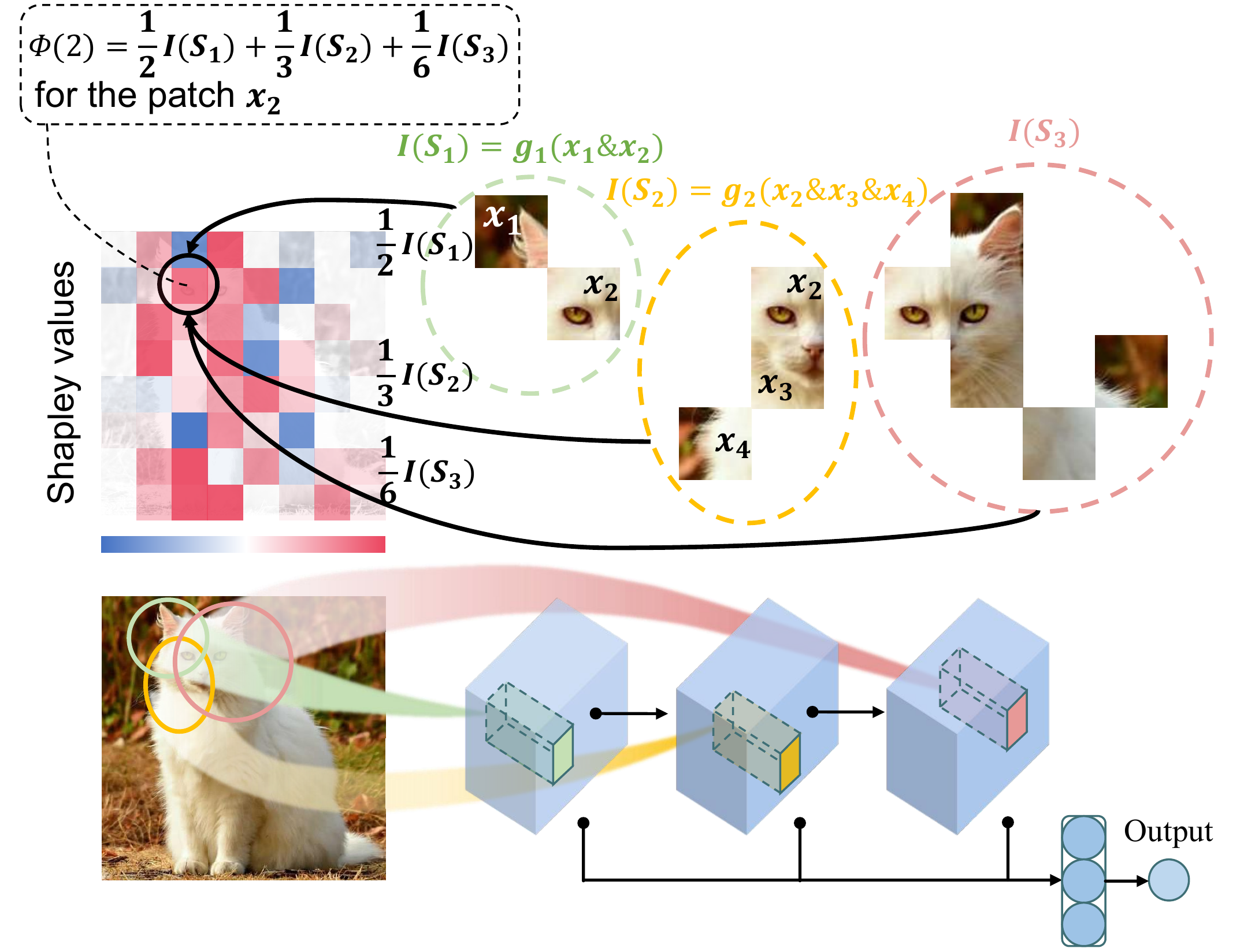}}
%\vskip -0.1in
\caption{Overview of the HarsanyiNet. The HarsanyiNet encodes different Harsanyi interactions, each representing an AND relationship between different patches. Shapley values can be computed as re-allocation of Harsanyi interactions.}
\label{Fig:Harsanyiarchitecture}
\end{center}
\vskip -0.4in
\end{figure}

However, using the Shapley value in real-world applications is often impractical because (1) computing the exact Shapley value is NP-hard, and (2) existing approximation techniques~\citep{castro2009polynomial,lundberg2017unified} often confront a dilemma in that approximating Shapley values with an acceptable accuracy usually requires to conduct a huge number of network inferences.

Thus, in this paper, we aim to directly jump out of the above dilemma and design a neural network, namely \textit{HarsanyiNet}, which simultaneously conducts model inference on the input sample and computes the exact Shapley value for each input variable in a single forward propagation\footnote{\url{https://github.com/csluchen/harsanyinet}}. 

Specifically, the theoretical foundation for the HarsanyiNet is that the Shapley value of an input variable can be reformulated as a redistribution of its different Harsanyi interactions \cite{harsanyi1963simplified} encoded by the DNN. Formally, given a DNN and an input sample with $n$ variables, a Harsanyi interaction $S$ represents an AND relationship between the variables in $S$, which is encoded by the DNN. A DNN usually encodes many different Harsanyi interactions. Each Harsanyi interaction makes a specific numerical contribution, denoted by $I(S)$, to the inference score of the DNN. Let us take the interaction between image patches $S\!=\!\{\textit{eye},\textit{nose}, \textit{mouth}\}$ as a toy example. If all these patches co-appear, then they form a face pattern and make a specific interaction effect $I(S)$ to the confidence score of face detection. Masking any patch will destroy this pattern and remove the interaction effect, \ie, making $I(S)\!=0\!$.

Because it is proven that the Shapley value can be computed using Harsanyi interactions, the activation of each intermediate neuron in the HarsanyiNet is designed to represent a specific Harsanyi interaction. The intermediate neuron is termed as the \textit{Harsanyi unit}. Such a network design enables us to derive the exact Shapley value of an input variable using activation scores of Harsanyi units.

The proposed HarsanyiNet has significant advantages over existing approaches for approximating Shapley values by conducting a single network inference.

$\bullet$ Existing approximation methods,~\eg, DeepSHAP~\cite{lundberg2017unified} and FastSHAP~\cite{jethani2021fastshap}, estimate Shapley values with considerable errors, but the HarsanyiNet can generate \textbf{exact} Shapley values, which is both theoretically guaranteed and experimentally verified.  

$\bullet$ The only existing work allowing to compute accurate Shapley values in a single forward propagation is the ShapNet~\cite{wang2021shapley}. 
%\textcolor{red}{However, we can prove that the ShapNet can be explained as a special case of the HarsanyiNet, although architectures of the two DNNs are dramatically different.}  
However, the ShapNet is designed to only encode interactions between at most $k$ input variables ($k\ll n$, they set $k=4$), and the computational cost of the ShapNet's inference (forward propagation) is $2^k$ times more than that of traditional networks. Alternatively, this study also provides another network (\ie, Deep ShapNet) to encode more complex interactions, but it cannot guarantee the accuracy of the computed Shapley values. In comparison, the HarsanyiNet does not limit the number of the input variables within the interaction, thereby ensuring broader applicability and exhibiting significantly better performance.

%the computational cost of the ShapNet's one forward propagation is $\mathcal{O}(2^k)$, where $k$ is the number of input variables in the interactions encoded by the network. Therefore, the ShapNet was limited to encode interactions between a small number of variables, \ie, $k\ll n$. Otherwise, encoding interactions between more variables would either lead to a very high computational complexity or would no longer guarantee the accuracy of the computed Shapley values.

Moreover, we implement two specific HarsanyiNets in the experiment, the \textit{Harsanyi-MLP} extended from multi-layer perceptrons (MLP) and the \textit{Harsanyi-CNN} developed on convolutional neural networks (CNN). 

The contributions of this paper can be summarized as follows. (1) We propose a novel neural network architecture, the \textit{HarsanyiNet}, which can simultaneously perform model inference and compute exact Shapley values in one forward propagation. (2) Following the paradigm of HarsanyiNet, we design Harsanyi-MLP and Harsanyi-CNN for tabular data and image data, respectively. (3) The HarsanyiNet does not constrain the representation of specific interactions, but it can still guarantee the accuracy of Shapley values.

\section{Related Work}
\label{sec:relatedwork}
Estimating the importance\slash attribution\slash saliency of input variables represents a typical direction in XAI. In general, previous attribution methods usually computed the attributions of input variables based on gradient~\citep{simonyan2013deep, springenberg2014striving, shrikumar2016not, selvaraju2017grad,sundararajan2017axiomatic}, via back-propagation of attributions~\cite{bach2015pixel, montavon2017explaining, shrikumar2017learning}, and based on perturbations on the input variables~\cite{ ribeiro2016should, zintgraf2017visualizing, fong2017interpretable, lundberg2017unified, plumb2018model,ian2021explaining, deng2021mutual, chen2022algorithms}.

\subsection{Shapley values}
Unlike other attribution methods, the Shapley value is designed in game theory. Let us consider the following cooperative game, in which a set of $n$ players $N=\{1,2,\ldots,n\}$ collaborate with each other and finally win a reward $R$. The Shapley value~\citep{shapley1953npersongame} is then developed as a fair approach to allocating the overall reward $R$ to the $n$ players. The Shapley value $\phi(i)$ is defined as the compositional reward allocated from $R$ to the player $i\in N$, and we can consider $\phi(i)$ reflects the numerical contribution of the player $i$ in the game.

\begin{definition}[Shapley values] \label{def:shapley}
    The Shapley value $\phi(i)$ of an input variable $i$ is given by
    \begin{equation}\label{eq:shapley}
    \phi(i)\coloneqq\!\!\!\!\!\!\sum_{S\subseteq N\setminus\{i\}}\!\!\!\!\frac{|S|!(n-|S|-1)!}{n!} \left[V(S \cup \{i\} ) \!\!-\!\! V(S)\right],
    \end{equation}
     where $V: 2^N \mapsto \mathbb{R}$ denotes the reward function, \ie, $\forall S\subseteq N, V(S)$ measures the reward if a subset $S$ of players participate in the game. Thus, $V(\emptyset)=0$, and $V(N)=R$ denotes the overall reward won by all players in $N$.
\end{definition}

\textbf{Faithfulness.} The Shapley value is widely considered a relatively faithful attribution, because~\citet{young1985monotonic} has proved that it is the unique game-theoretic solution that satisfies the \textit{linearity}, \textit{dummy}, \textit{symmetry} and \textit{efficiency} axioms for ideal attributions. Please see Appendix~\ref{sec:appendix_shapley} for details.

In addition to estimating the importance of each input variable, 
the Shapley value is also widely used to estimate the importance of every single data point in a whole dataset, which can be, for instance, used to address data evaluation problem \citep{jia2019efficient, jia2021scalability}.

\subsection{Dilemma of computational complexity versus approximation accuracy}
The biggest challenge of applying Shapley values to real-world applications is the NP-hard computational complexity. According to \cref{eq:shapley} and Section~\ref{Methodology}, when we compute Shapley values for input variables of a DNN, it requires to conduct inference on all the $2^n$ different masked samples. To alleviate the computational burden, many approximation methods \citep{castro2009polynomial,lundberg2017unified,covert2021improving} have been proposed. However, as Figure~\ref{Fig:convergence_tabular} shows, a higher approximation accuracy of Shapley values usually requires more network inferences (\eg, inferences on as many as thousands of masked samples).

Specifically, some approaches estimated Shapley values via sampling techniques~\cite{castro2009polynomial, strumbelj2010efficient, okhrati2021multilinear, mitchell2022sampling}, and some converted the approximation of Shapley values to a weighted least squares problem~\cite{lundberg2017unified, simon2020projected, covert2021improving}. However, these methods all faced a dilemma,~\ie, a more accurate approximation required higher computational costs. 

Other studies accelerated the computation by assuming a specific distribution of data~\cite{chen2018shapley}, ignoring small interactions between input variables~\cite{wang2022accelerating}, or learning an explainer model to directly predict the Shapley value~\cite{jethani2021fastshap}. However, these methods could not generate fully accurate Shapley values. \citet{wang2021shapley} proposed the ShapNets. The ShapNet was constrained to only encode interactions between a small number of variables (usually less than 4). When the ShapNet was extended to encode interactions between more variables, it could no longer estimate exactly accurate Shapley values. In comparison, our HarsanyiNet can accurately compute Shapley values in a single forward propagation, and it is not constrained to encode specific interactions, thereby exhibiting much more flexibility and better performance.

\section{Methodology}
\label{Methodology}

 The Shapley value defined in~\cref{eq:shapley} is widely used to estimate attributions of input variables in a DNN. We consider the DNN as a cooperative game, and consider input variables $\mathbf{x}=[x_1,x_2,...,x_n]^\intercal$ as players, $N=\{1,2,\dots,n\}$. $v(\mathbf{x}) \in \mathbb{R}$ corresponds to the network prediction score\footnote{As in previous studies~\cite{jethani2021fastshap,wang2022accelerating}, if the network has a scalar output, then $v(\mathbf{x})$ can be formulated directly as the network output. If the network has a vector output, \eg, multi-category classification, we may define $v(\mathbf{x})$  as the output dimension corresponding to the ground-truth category.\label{fn:fn1}} on $\mathbf{x}$. Let $\mathbf{x}_S$ denote a masked sample, where input variables in $N \setminus S, S\subseteq N$ are masked by baseline values\footnote{The baseline value can be set as zero, mean value over different inputs or other statistic values in previous studies~\cite{lundberg2017unified, covert2021improving}.}. 
In this way, we can define the total reward gained by the input variables in $S$ as the inference score on the masked sample $\mathbf{x}_S$, \ie, $V(S)\coloneqq v(\mathbf{x}_S) - v(\mathbf{x}_\emptyset)$. Thus, the Shapley value $\phi(i)$ in~\cref{eq:shapley} measures the importance of the $i$-th input variable to the network prediction score.

\subsection{Preliminaries: Harsanyi interactions}
\label{Methodology:preliminary}
The Harsanyi interaction (or the Harsanyi dividend)~\citep{harsanyi1963simplified} provides a deeper insight into the essential reason why the Shapley value is computed as in~\cref{eq:shapley}. A DNN usually does not use each individual input variable to conduct model inference independently. Instead, the DNN models the interactions between different input variables and considers such interactions as basic inference patterns. To this end, the Harsanyi interaction $I(S)$ measures the interactive effect between each subset $S\subseteq N$ of input variables, which is encoded by a DNN.

\begin{definition}[Harsanyi interactions]\label{def:harsanyi}
    The Harsanyi interaction between a set of variables in $S$ \wrt~the model output $v(\mathbf{x})$ is recursively defined $I(S) \coloneqq V(S)- \sum\nolimits_{L\subsetneq S} I(L) =  v(\mathbf{x}_S) - v(\mathbf{x}_\emptyset)- \sum\nolimits_{L\subsetneq S} I(L)$ subject to $I(\emptyset) \coloneqq 0.$
\end{definition}

According to Definition~\ref{def:harsanyi}, the network output can be explained as the sum of all Harsanyi interactions, \ie, $v(\mathbf{x})$$-v(\mathbf{x}_\emptyset)$$=\sum_{S\subseteq N} I(S)$. Essentially, each Harsanyi interaction $I(S)$ reveals an AND relationship between all the variables in set $S$. Let us consider a visual pattern $S=\{\textit{eye},\textit{nose},\textit{mouth}\}$ for face detection as a toy example. If the image patches of \textit{eye}, \textit{nose}, and \textit{mouth} appear together, the co-appearance of the three parts forms a visual pattern and makes a numerical contribution $I(S)$ to the classification score $v(\mathbf{x})$ of the face. Otherwise, masking any part in $S$ will deactivate the pattern and remove the interactive effect, \ie, making $I(S)=0$. 

\citet{grabisch2016set} and \citet{ren2023defining} further proved that the Harsanyi interaction also satisfies the four properties, namely \textit{linearity}, \textit{dummy}, \textit{symmetry} and \textit{efficiency}.

\subsection{Harsanyi interactions compose Shapley values}
\label{Methodology:redefine}
We jump out of the dilemma of computational complexity versus approximation accuracy mentioned in Section~\ref{sec:relatedwork}. Theorem~\ref{theorem1} allows us to derive a novel neural network architecture, namely \textit{HarsanyiNet}, which uses Harsanyi interactions to simultaneously perform model inference and compute exact Shapley values in a single forward propagation. 

\textbf{$\bullet$ Basic requirements for the HarsanyiNet.} The key idea of the HarsanyiNet is to let different intermediate-layer neurons to represent different Harsanyi interactions. Later, we will prove that we can use such a network design to compute the exact Shapley values in a single forward propagation. Specifically, let us introduce the following two designs in the HarsanyiNet. 

Firstly, as shown in Figure~\ref{Fig:Harsanyiarchitecture}, the HarsanyiNet has $L$ cascaded blocks in the neural network. In each block, we add an AND operation layer between a linear layer and a ReLU layer. Given an input sample $\mathbf{x}$, let $z_u^{(l)}(\mathbf{x})$ denote the $u$-th dimension of the output feature vector $\mathbf{z}^{(l)}(\mathbf{x}) \in \mathbb{R}^{m^{(l)}}$ in the $l$-th linear layer. The HarsanyiNet is designed to let each feature dimension $z_u^{(l)}(\mathbf{x})$ satisfy the following two requirements, and $z_u^{(l)}(\mathbf{x})$ is also called a \textit{Harsanyi unit}. Theorem~\ref{thm:calculateShapley} will show how to use the Harsanyi unit to compute exact Shapley values directly.

\begin{requirement}[\textbf{R1}]\label{requirement:1}
    The neural output $z_u^{(l)}(\mathbf{x})$ is exclusively determined by a specific set of input variables $\mathcal{R}_u^{(l)} \subseteq N$,  namely the \textit{receptive field} of neuron $z_u^{(l)}(\mathbf{x})$. In other words, none of the other input variables in $N\setminus \mathcal{R}_u^{(l)}$ affect the neuron activation, \ie, given two arbitrary samples $\mathbf{x}$ and $\mathbf{x}'$, if $~\forall~i\in \mathcal{R}_u^{(l)},$ $x'_i=x_i$, then $z_u^{(l)}(\mathbf{x}') = z_u^{(l)}(\mathbf{x})$.
\end{requirement}

\begin{requirement}[\textbf{R2}]\label{requirement:2}
Masking any variables in the receptive field $\mathcal{R}_u^{(l)}$  of the neuron $z_u^{(l)}$ will make $z_u^{(l)}(\mathbf{x})=0$. Specifically, let $\mathbf{x}_S$ denote the sample obtained by masking variables in the set $N\setminus S$ in the sample $\mathbf{x}$. Then, given any masked sample $\mathbf{x}_S$, the neuron $z_u^{(l)}$ must satisfy the property  $z_u^{(l)}(\mathbf{x}_S) = z_u^{(l)}(\mathbf{x})\cdot \prod_{i\in \mathcal{R}_u^{(l)}}\mathbbm{1}(i \in S)$.
    
\end{requirement}

\textbf{These two requirements indicate that a Harsanyi unit $z_u^{(l)}$ must represent an AND relationship between input variables in $\mathcal{R}_u^{(l)}$.} Changing variables outside the receptive field $\mathcal{R}_u^{(l)}$ will not affect the neural activation of the Harsanyi unit, \ie, $z_u^{(l)}(\mathbf{x}_S)=z_u^{(l)}(\mathbf{x})$, but masking any variables in $\mathcal{R}_u^{(l)}$ will deactivate the unit, \ie, making $z_u^{(l)}(\mathbf{x}_S) = 0$.

Secondly, although the HarsanyiNet may have various types of outputs (including a scalar output, a vectorized output, a matrix output, and a tensor output), \textbf{each dimension of the network output is designed as a weighted sum of all Harsanyi units}. Let $\mathbbm{v}(\mathbf{x})$ denote the multi-dimensional output, and let $v(\mathbf{x})$ be an arbitrary output dimension parameterized by $\{\mathbf{w}_v^{(l)}\}$. Then, we get 
\begin{equation}\label{eq:efficiency}\!\!\!\!\mathbbm{v}(\mathbf{x})\!=\![v(\mathbf{x}), v'(\mathbf{x}),\dots]^\intercal, \quad \!\! v(\mathbf{x})\! = \!\! \sum_{l=1}^{L}(\mathbf{w}_v^{(l)})^\intercal \mathbf{z}^{(l)}(\mathbf{x}), \!\!
\end{equation} where $\mathbf{w}_v^{(l)}\in \mathbb{R}^{m^{(l)}}$ denotes the weight for the specific output dimension $v$. As shown in Figure~\ref{Fig:Harsanyiarchitecture}, the above equation can be implemented by adding skip connections to connect the Harsanyi units in all $L$ layers to the HarsanyiNet output.

\textbf{$\bullet$ Proving that we can compute accurate Shapley values in a single forward propagation.} The preceding paragraphs only outline the two requirements for Harsanyi units, and Section~\ref{sec:construct} will introduce how to force neurons to meet such requirements. Before that, we derive Theorem~\ref{thm:calculateShapley} to prove that the above requirements allow us to compute the exact Shapley values in a forward propagation.

\begin{theorem}[Connection between Shapley values and Harsanyi interactions, proof in \citep{harsanyi1963simplified}] \label{theorem1} The Shapley value $\phi(i)$ equals to the sum of evenly distributed Harsanyi interactions that contain $i$, \textit{i.e.}, 
\begin{equation}\label{eq:HarsanyiShapley}
    \phi(i) = \sum\nolimits_{S\subseteq N: S \ni i} \frac{1}{|S|}I(S).
\end{equation}
\end{theorem}

Theorem~\ref{theorem1} demonstrates that we can understand the Shapley value $\phi(i)$ as a uniform reassignment of each Harsanyi interaction $I(S)$ which includes the variable $i$. 
For example, let us consider a toy model that uses \textit{age (a)}, \textit{education (e)}, \textit{occupation (o)}, and \textit{marital status (m)} to infer the income level. We assume that we can only decompose four non-zero Harsanyi interactions, \ie, $v(\mathbf{x})\!=\! \sum_{S\subseteq N=\{a,e,o,m\}}I(S)\!=\!I(\{a,o\})+I(\{a,e\}) + I(\{a,o,m\}) + I(\{o,m\})$ to simplify the story. We uniformly allocate the numerical contribution $I(\{a,o, m\})$ to variables \textit{age}, \textit{occupation}, and \textit{marital status}, with each receiving $\frac{1}{3}I(\{a,o,r\})$ as a component of its attribution. In this way, each input variable accumulates compositional attributions from different Harsanyi interactions, \eg, $\hat{\phi}(a)=\frac{1}{2}I(\{a,o\})+\frac{1}{2}I(\{a,e\}) +\frac{1}{3}I(\{a,o,m\})$. Such an accumulated attribution $\hat{\phi}(a)$ equals to the Shapley value $\phi(a)$. 

\begin{lemma}[Harsanyi interaction of a Harsanyi unit, proof in Appendix~\ref{appendix:lemma}]\label{thm:J(S)}
    Let us consider the output of a Harsanyi unit $z_u^{(l)}(\mathbf{x})$ as the reward. Then, let $J_u^{(l)}(S)$ denote the Harsanyi interaction \wrt~the function $z_u^{(l)}(\mathbf{x})$. Then, we have $J_u^{(l)}(\mathcal{R}_u^{(l)})=z_u^{(l)}(\mathbf{x})$, and $\forall S\neq \mathcal{R}_u^{(l)}, J_u^{(l)}(S) =0$, according to the two requirements R\ref{requirement:1} and R\ref{requirement:2}.
\end{lemma}

\begin{theorem}[Proof in Appendix~\ref{appendix:theorems}] \label{thm:linearity}
     Let a network output $v(\mathbf{x})\in \mathbb{R}$ be represented as $v(\mathbf{x}) = \sum\nolimits_{l=1}^{L}(\mathbf{w}_v^{(l)})^\intercal \mathbf{z}^{(l)}(\mathbf{x})$, according to~\cref{eq:efficiency}. In this way, the Harsanyi interaction between input variables in the set $S$ computed on the network output $v(\mathbf{x})$ can be represented as $I(S) = \sum\nolimits_{l=1}^{L}\sum\nolimits_{u=1}^{m^{(l)}}w_{v,u}^{(l)}J_u^{(l)}(S).$
\end{theorem}
Theorem~\ref{thm:linearity} shows that the Harsanyi interaction $I(S)$~\wrt~network output $v(\mathbf{x})$ can be represented as the sum of Harsanyi interactions $J_u^{(l)}(S)$  computed on different Harsanyi units $z^{(l)}_u(\mathbf{x})$. In this manner, we plug the conclusions in Lemma~\ref{thm:J(S)} and Theorem~\ref{thm:linearity} into~\cref{eq:HarsanyiShapley}, and we derive the following theorem.

\begin{theorem}[\textbf{Deriving Shapley values from Harsanyi units in intermediate layers}, proof in Appendix~\ref{appendix:theorems}]\label{thm:calculateShapley}
    The Shapley value $\phi(i)$ can be computed as
    \begin{equation}\label{eq:harsanyiunitshapley}
        \phi(i)\! = \!\sum\nolimits_{l=1}^{L}\sum\nolimits_{u=1}^{m^{(l)}} \frac{1}{|\mathcal{R}_u^{(l)}|} w_{v,u}^{(l)} z_u^{(l)}(\mathbf{x})\mathbbm{1}(\mathcal{R}_u^{(l)}\ni i).
        %\phi(i) = \!\!\!\!\!\sum_{S \subseteq N, S\ni i}\frac{1}{|S|}I(S) = \!\!\!\!\!\sum_{S \subseteq N, S\ni i}\frac{1}{|S|}\sum_uw_u J_u(S).
    \end{equation}
\end{theorem}
Theorem~\ref{thm:calculateShapley} demonstrates that the Shapley value $\phi(i)$ can be directly computed using the outputs of Harsanyi units $\mathbf{z}^{(l)}(\mathbf{x})$ in the intermediate layers in forward propagation.

\textbf{Cost of computing HarsanyiNet.} We conduct one network inference to obtain the outputs of all Harsanyi units $z_u^{(l)}(\mathbf{x})$. Then, we compute Shapley values based on~\cref{eq:harsanyiunitshapley}, whose computational cost is $\mathcal{O}(nM)$, where {\small{$M\!=\!\sum_{l=1}^Lm^{(l)}$}} denotes the total number of Harsanyi units. The computational cost $\mathcal{O}(nM)$ is negligible, compared to the heavy computational cost of one forward propagation. \textbf{Therefore, we can roughly consider the overall cost of computing Shapley values as one forward propagation.}

\subsection{Designing the HarsanyiNet towards R\ref{requirement:1} and R\ref{requirement:2}} 
\label{sec:construct}
This subsection first introduces the detailed design of the HarsanyiNet. The basic idea is to construct a neural network in which each neuron represents an AND relationship between its children nodes in the previous layers, and the neuron's receptive field can be computed as the union of the receptive fields of its children nodes. Then, Theorem~\ref{thm:implementation} proves that such a network design satisfies the requirements R\ref{requirement:1} and R\ref{requirement:2} in Section~\ref{Methodology:redefine}.

\textbf{Harsanyi blocks.} As Figure~\ref{Fig:Harsanyiarchitecture} shows, the HarsanyiNet contains $L$ cascaded \textit{Harsanyi blocks.} Specifically, we use tuple $(l,u)$ to denote the $u$-th neuron in the $l$-th Harsanyi block's linear layer. Each neuron $(l,u)$ has a set of children nodes $\mathcal{S}_u^{(l)}$. The children nodes in $\mathcal{S}_u^{(l)}$ can be selected from all neurons in all ($l\!-\!1$) previous Harsanyi blocks\footnote{In particular, children nodes in  $\mathcal{S}_u^{(1)}$ are directly selected from the input variables.}. Alternatively, we can just select children nodes from the $(l\!-\!1)$-th block, as a simplified implementation. The children nodes $\mathcal{S}_u^{(l)}$ which can be learned for each neuron $(l,u)$ will be introduced later. Thus, given the children set $\mathcal{S}_u^{(l)}$, the neural activation $z_u^{(l)}(\mathbf{x})$ of the neuron $(l, u)$ is computed by applying the linear, AND, and ReLU operations
\begin{align}
\label{eq:implementation1}
     g^{(l)}_u(\mathbf{x}) &=\!\!(\mathbf{A}^{(l)}_u)^\intercal\!\!\cdot\!\!\left(\mathbf{\Sigma}_u^{(l)} \!\cdot \!\mathbbm{z}^{(l-1)}\right).\begin{array}{r}//\footnotesize{\text{Linear operation}}  \\
         \footnotesize{\text{on children nodes}}  
    \end{array} \\
    \label{eq:implementation2}
    h^{(l)}_u(\mathbf{x}) &= g^{(l)}_u(\mathbf{x}) \!\cdot\!\!\!\!\!\!\!\!\prod_{(l',u')\in\mathcal{S}_u^{(l)}}\!\!\!\!\!\!\!\mathbbm{1}(z^{(l')}_{u'}(\mathbf{x})\neq 0).~~~~\begin{array}{r}
         //\footnotesize{\text{AND}}  \\
         \footnotesize{\text{operation}}  
    \end{array} \\
    \label{eq:implementation3}
     z^{(l)}_u(\mathbf{x}) &=\text{ReLU}( h^{(l)}_u(\mathbf{x})).~~~~~\begin{array}{r}
         //\footnotesize{\text{Non-linear operation}} 
    \end{array}
\end{align}

In the above equations, $\mathbbm{z}^{(l-1)} = [\mathbf{z}^{(1)}(\mathbf{x})^\intercal, \mathbf{z}^{(2)}(\mathbf{x})^\intercal,\dots, \allowbreak \mathbf{z}^{(l-1)}(\mathbf{x})^\intercal]^\intercal \in \mathbb{R}^{M^{(l)}}$ vectorizes the neurons in all the previous $(l-1)$ blocks. The children set $\mathcal{S}_u^{(l)}$ is implemented as a binary diagonal matrix $\mathbf{\Sigma}_u^{(l)} \in \{0,1\}^{M^{(l)}\times M^{(l)}}$, which selects children nodes of the neuron $(l,u)$ from all $M^{(l)}=\sum_{l'=1}^{l-1}m^{(l')}$ neurons in all the $(l-1)$ previous blocks. 
$\mathbf{A}_u^{(l)} \in \mathbb{R}^{M^{(l)}}$ denotes the weight vector.

The three operations in~\cref{eq:implementation1,eq:implementation2,eq:implementation3} ensure that each Harsanyi unit $z_u^{(l)}(\mathbf{x})$ represents an AND relationship among its children nodes (more discussions in Appendix~\ref{appendix: discussion}). 

To implement the children selection in~\cref{eq:implementation1}, we compute the binary diagonal matrix $\mathbf{\Sigma}_u^{(l)}$ by setting $(\mathbf{\Sigma}_u^{(l)})_{i,i}\!\!=\!\!\mathbbm{1}((\boldsymbol{\tau}_u^{(l)})_i>0)$, where $\boldsymbol{\tau}_u^{(l)} \in \mathbb{R}^{M^{(l)}}$ is a trainable parameter vector. Note that during the training phase, the gradient of the loss function cannot pass through $\mathbf{\Sigma}_u^{(l)}$ to $\boldsymbol{\tau}_u^{(l)}$ in the above implementation; therefore, we employ Straight-Through Estimators (STE)~\citep{bengio2013estimating} to train the parameter $\boldsymbol{\tau}_u^{(l)}$. The STE uses $(\mathbf{\Sigma}_u^{(l)})_{i,i}\!\!\!=\!\!\!\mathbbm{1}((\boldsymbol{\tau}_u^{(l)})_i>0)$ in the forward-propagation and set $\partial(\mathbf{\Sigma}_u^{(l)})_{i,i}/\partial(\boldsymbol{\tau}_u^{(l)})_i\!=\!\beta e^{-(\boldsymbol{\tau}_u^{(l)})_i}/(1+e^{-(\boldsymbol{\tau}_u^{(l)})_i})^2$ in the back-propagation process, where $\beta$ is a positive scalar. Besides, to reduce the optimization difficulty of the AND operation in~\cref{eq:implementation2}, we approximate the AND operation as
{\small$h_u^{(l)}(\mathbf{x})\!\!=\small g_u^{(l)}(\mathbf{x})$ $\left[\prod_{u'=1}^{M^{(l)}}(\mathbf{\Sigma}_u^{(l)}\!\cdot\! \tanh(\gamma\! \cdot\! \mathbf{\Sigma}_u^{(l)} \mathbbm{z}^{(l-1)}) \!+\! (\mathbf{I}\!-\! \mathbf{\Sigma}_u^{(l)})\!\cdot\! \mathbf{1})_{u'}\right]^{1/\mathbf{tr}(\mathbf{\Sigma}_u^{(l)})}$}, where $\gamma$ is a positive scalar. Here, each output dimension of the function $\tanh(\cdot)$ is within the range of $[0,1)$, since $\mathbbm{z}^{(l\!-\!1)}$ passes through the ReLU operation, and $\forall u, \mathbbm{z}^{(l\!-\!1)}_u\!\geq\! 0$.  

\textbf{Input and receptive field.} Let us set $\mathbf{z}^{(0)}=\mathbf{x}-\mathbf{b} \in \mathbb{R}^n$ as the input of the linear operation in the first Harsanyi block, and let us define the baseline value $b_i$ as the masking state of each input variable $x_i$. To further simplify the implementation, we adopt a single value baseline $\mathbf{b}$\footnote{We simply set the baseline value $\mathbf{b} = \mathbf{0}$, since we use the ReLU function as the non-linear operation.}. It is worth noting that more sophisticated baseline values have been discussed in~\citep{lundberg2017unified,covert2020understanding, sundararajan2020many,chen2022algorithms}. 
Based on~\cref{eq:implementation1,eq:implementation2,eq:implementation3}, we can obtain that the receptive field $\mathcal{R}^{(l)}_u$ of a Harsanyi unit $(l, u)$ can be computed recursively, as follows.
    \begin{equation}\label{eq:rf}
        \mathcal{R}^{(l)}_u := \cup_{(l',u')\in \mathcal{S}^{(l)}_{u}}\mathcal{R}^{(l')}_{u'}, \quad \st~\mathcal{R}^{(1)}_u := \mathcal{S}^{(1)}_u.
    \end{equation}

\begin{theorem}[proof in Appendix~\ref{appendix:theorems}]\label{thm:implementation}
Based on~\cref{eq:implementation1,eq:implementation2,eq:implementation3}, the receptive field $\mathcal{R}^{(l)}_u$ of the neuron $z_u^{(l)}$ automatically satisfies the two requirements R\ref{requirement:1} and R\ref{requirement:2}.
\end{theorem}
This theorem proves that setting each neuron $z_u^{(l)}$ based on~\cref{eq:implementation1,eq:implementation2,eq:implementation3} can successfully encode an AND relationship between input variables in $\mathcal{R}^{(l)}_u$. In other words, only the input variables in the receptive field $\mathcal{R}_u^{(l)}$ can affect neural output $z^{(l)}_u(\mathbf{x})$, and masking any input variables in $\mathcal{R}_u^{(l)}$ will make $z_u^{(l)}(\mathbf{x}) = 0$. 
In particular, let us consider the inference on a masked sample $\mathbf{x}_S$ as an example. According to Theorem~\ref{thm:implementation}, the masked sample $\mathbf{x}_S$ is implemented by setting $\forall i\notin S, x_i=b_i$. Subsequently, this masked sample can exclusively activate all Harsanyi units subject to $\mathcal{R}_u^{(l)}\subseteq S$. All other Harsanyi units are not activated, \ie, $\forall \mathcal{R}_u^{(l)}\not\subseteq S, z_u^{(l)}(\mathbf{x}_S)=0$.

\begin{figure}[ht!]
\begin{center}
\centerline{\includegraphics[width=\columnwidth,trim={0cm 2cm 0cm 2cm},clip]{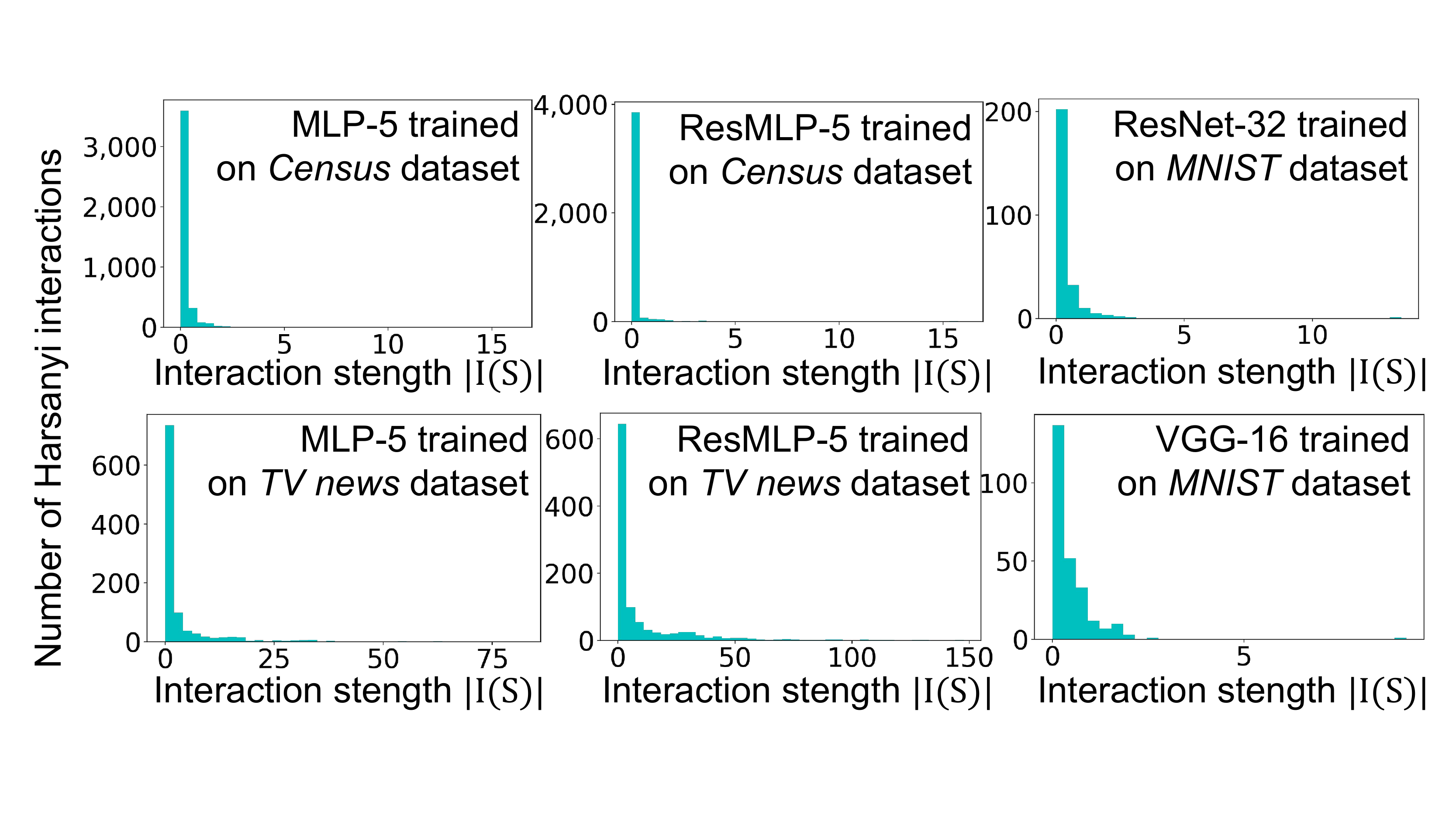}}
\vskip -0.05in
\caption{The histogram of interaction strength $|I(S)|$ of different Harsanyi interactions encoded by a DNN.}
\label{fig:sparsity}
\vskip -0.2in
\end{center}
\end{figure}
\subsection{Discussion on the sparsity of Harsanyi interactions}

Theoretically, a model can encode at most $2^n$ different Harsanyi interactions, but each AI model has its own limitation in encoding interactions, which are far less than $2^n$. The HarsanyiNet learns at most $M\!=\!\sum_{l=1}^Lm^{(l)}$ Harsanyi interactions, as proved in Lemma~\ref{thm:J(S)} and Theorem~\ref{thm:linearity}. Thus, the next question is whether the HarsanyiNet has sufficient representation capacity to handle real-world applications.

To this end, recent studies have observed~\cite{deng2021discovering, ren2023defining, li2023does} and mathematically proved~\cite{ren2023where} that traditional DNNs often encode only a few Harsanyi interactions in real-world applications, instead of learning all $2^n$ Harsanyi interactions. To be precise, the network output can be represented as 
\begin{equation}
    v(\mathbf{x})= \sum\nolimits_{S\in 2^N}I(S)= \sum\nolimits_{S\in \Omega}I(S) + \epsilon,
\end{equation}where {\small {$\Omega \!\subseteq \!2^N \!\!= \!\!\{S' \!\subseteq \!N\}$}} denotes a small set of Harsanyi interactions with considerable interaction strength $|I(S)|$. All other Harsanyi interactions have negligible interaction strength, \ie, $|I(S)|\approx 0$, which can be considered noisy inference patterns. {\small$\epsilon\!=\! \sum_{S'\in 2^N\setminus \Omega}I(S')$} is relatively small.

Furthermore, we conducted new experiments to verify the sparsity of Harsanyi interactions in DNNs. Given a trained network $v$ and an input sample $\mathbf{x}$, we computed the interaction strength $|I(S)|$ of all $2^n$ Harsanyi interactions \wrt all $S\subseteq N$. We followed~\cite{ren2023defining} to compute the interaction strength $|I(S)|$ of different Harsanyi interactions\footnote{For image data,~\citet{ren2023defining} computed Harsanyi interactions between randomly sampled image regions to reduce the computational cost.}. Please see Appendix~\ref{appendix:aog} for more details. Figure~\ref{fig:sparsity} shows the extracted Harsanyi interactions. Such experiments were conducted on various DNNs, including MLP, the residual MLP\footnote{We used 5-layer MLP (MLP-5) and 5-layer residual MLP (ResMLP-5) with 100 neurons in each hidden layer respectively.} used in~\citet{touvron2022resmlp}, the residual net with 32 layers (ResNet-32)~\cite{he2016deep} and the VGG net with 16 layers (VGG-16)~\cite{simonyan2013deep}, on the Census Income~\cite{Dua:2019}, the TV news commercial detection~\cite{Dua:2019}, and the MNIST~\cite{LeCun:2010} datasets. Only a few Harsanyi interactions were found to be salient. Most Harsanyi interactions were close to zero, and could be considered noise.

Therefore, the above experiments demonstrated that many applications only required DNNs to encode a few salient Harsanyi interactions, instead of modeling an exponential number of Harsanyi interactions. From this perspective, the HarsanyiNet has sufficient representation capacity. 
\section{Experiments}
\label{sec:experiment}
\subsection{Two types of HarsanyiNets} 
\label{sec:experimentsetting}
In the experiments, we constructed and tested two types of HarsanyiNets following the paradigm in~\cref{eq:implementation1,eq:implementation2,eq:implementation3},~\ie, the HarsanyiNet constructed with fully-connected layers, namely \textit{Harsanyi-MLP}, and the HarsanyiNet constructed with convolutional layers, namely \textit{Harsanyi-CNN}. The Harsanyi-MLP was suitable for handling tabular data, and the Harsanyi-CNN was designed for image data. 

$\bullet$ \textbf{Harsanyi-MLP} was designed as an extension of the MLP network. As mentioned at the beginning of Section~\ref{sec:construct}, we chose not to connect each neuron $(l,u)$ from the neurons in all $(l-1)$ previous blocks. Instead, we simply selected a set of children nodes $\mathcal{S}_u^{(l)}$  from the $(l-1)$-th block. Specifically, this was implemented by fixing all elements in $\boldsymbol{\tau}_u^{(l)}$ corresponding to all neurons in the 1st, 2nd,$\dots$,$(l-2)$-th blocks to 0.

$\bullet$ \textbf{Harsanyi-CNN} mainly used the following two specific settings to adapt convolutional layers into the paradigm of HarsanyiNet. 
\textbf{Setting 1.} Similar to the Harsanyi-MLP, the Harsanyi-CNN constructed the children set $\mathcal{S}_u^{(l)}$ of each neuron $(l,u)$ from the neurons in the $(l-1)$-th block. Let $C\times K\times K$ denote the tensor size of the convolutional kernel. As Figure~\ref{Fig:childrenset} shows, we selected children nodes $\mathcal{S}_u^{(l)}$ of the neuron $(l,u=(c,h,w))$ from neurons in the $C\times K\times K$ sub-tensor, which was clipped from the feature tensor of the $(l-1)$-th block and corresponded to the upper neuron $(l,u)$, where $c,h,w$ represent the location of the neuron $(l,u)$. Accordingly, we had $\boldsymbol{\tau}_u^{(l)}$ as a $CK^2$-dimensional vector.

\textbf{Setting 2.} Furthermore, we set all neurons $(l,u=(:,h,w))$ at the same location, but on different channels, to share the same children set $\mathcal{S}^{(l)}_{u=(:,h,w)}$ to reduce the number of parameters $\boldsymbol{\tau}_u^{(l)}$. This could be implemented by letting all neurons $(l,u=(1,h,w)),\dots,(l,u=(C,h,w))$ share the same parameter $\boldsymbol{\tau}_u^{(l)}$.
Based on the above design, we proved that all Harsanyi units $(l,u=(c,h,w))$ in the same location $(h,w)$ on different channels $(c=1,\dots,C)$ had the same receptive field $\mathcal{R}^{(l)}_{u=(:,h,w)}$ and contributed to the same Harsanyi interaction $I(S=\mathcal{R}^{(l)}_{u=(:,h,w)})$. Please see Appendix~\ref{appendix:Harsanyi-cnn} for the proof.
Therefore, we further considered neurons in the same location $(h,w)$ on different channels as a single Harsanyi unit. In this way,~\cref{eq:implementation2} could be rewritten as $h^{(l)}_u(\mathbf{x}) = g^{(l)}_u(\mathbf{x}) \cdot\prod_{(l-1,u')\in\mathcal{S}_u^{(l)}}\mathbbm{1}(\sum_{c=1}^C |z^{(l-1)}_{u'=(c,h,w)}(\mathbf{x})|\neq 0)$.

In the implementation, we first applied a convolutional layer, max-pooling layer, and ReLU layer on the input image to obtain the feature $\mathbf{z}^{(0)}$ in an intermediate layer. Subsequently, we regarded $\mathbf{z}^{(0)}$ as the input of the Harsanyi-CNN, instead of directly using raw pixel values as input variables. It was because using the ReLU operation enabled us to simply define $b_i=0$ as the baseline value (\ie, the masking state) for all the feature dimensions $\mathbf{z}^{(0)}$. In addition, according to Setting 2, we could consider the feature vector $\mathbf{z}^{(0)}_{(:,h,w)}$ at location $(h,w)$ as a single input variable, and we used $\mathbf{z}^{(0)}_{(:,h,w)}=\mathbf{0}$ to identify its masking state.

\subsection{Experiments and comparison}
\label{sec:experimentandcomparison}
\textbf{Dataset.}
We trained the Harsanyi-MLP on three tabular datasets from the UCI machine learning repository~\citep{Dua:2019}, including the Census Income dataset ($n=12$), the Yeast dataset ($n=8$) and the TV news commercial detection dataset ($n=10$), where $n$ denotes the number of input variables. For simplicity, these datasets were termed \textit{Census}, \textit{Yeast}, and  \textit{TV news}. We trained the Harsanyi-CNN on two image datasets: the MNIST dataset~\citep{LeCun:2010} and the CIFAR-10 dataset~\citep{krizhevsky2009learning}.

\textbf{Accuracy of Shapley values and computational cost.} We conducted experiments to verify whether the HarsanyiNet could compute accurate Shapley values in a single forward propagation. We evaluated both the accuracy and the time cost of calculating Shapley values. We computed the root mean squared error (RMSE) between the estimated Shapley values  $\boldsymbol{\phi}_\mathbf{x}$ and the true Shapley values, \ie, RMSE=$\mathbb{E}_\mathbf{x} [\frac{1}{\sqrt{n}}\vert\vert \boldsymbol{\phi}_\mathbf{x} - \boldsymbol{\phi}^*_\mathbf{x}\vert\vert]$, where the vector of ground-truth Shapley values $\boldsymbol{\phi}^*_\mathbf{x}$ on the sample $\mathbf{x}$ could be directly computed by following Definition~\ref{def:shapley} when $n\leq 16$. 

%$RMSE=$\sqrt{\mathbb{E}_\mathbf{x}[\vert\vert \boldsymbol{\phi}_\mathbf{x} - \boldsymbol{\phi}^*_\mathbf{x}\vert\vert]}$

%However, input samples in the MNIST and CIFAR-10 datasets contained much more input variables, so we used the sampling method~\cite{castro2009polynomial} \textcolor{blue}{to approximate the true Shapley value. We conducted network inferences on as many as }$2^{16}$ sampled masked instances $\mathbf{x}_S$, in order to \textcolor{blue}{obtain a relatively convincing approximation.}

\textbf{Comparing with approximation methods.} We compared the accuracy of Shapley values computed by the HarsanyiNet with those estimated by various approximation methods, including the sampling method~\citep{castro2009polynomial}, KernelSHAP~\citep{lundberg2017unified}, KernelSHAP with paired sampling (KernelSHAP-PS)~\citep{covert2021improving}, antithetical sampling~\citep{mitchell2022sampling}, DeepSHAP~\cite{lundberg2017unified} and FastSHAP~\citep{jethani2021fastshap}. The approximation methods also computed Shapley values on the HarsanyiNet for fair comparison. Figure~\ref{Fig:convergence_tabular} shows that many approximation methods generated more accurate Shapley values, when they conducted more inferences for approximation. The number of inferences was widely used~\cite{lundberg2017unified, ancona2019explaining} to quantify the computational cost of approximating Shapley values. These methods usually needed thousands of network inferences to compute the relatively accurate Shapley values. In comparison, the HarsanyiNet only needed one forward propagation to obtain the exact Shapley values (see ``$\star$'' in Figure~\ref{Fig:convergence_tabular}). DeepSHAP and FastSHAP could compute the Shapley values in one forward propagation, but as shown in Table~\ref{tab:fastanddeep}, the estimated errors of Shapley values were considerably larger than the HarsanyiNet.

\begin{table}[t]
\caption{Root mean squared errors of the estimated Shapley value and the classification accuracy of the DNN.}
\label{tab:tabularaccuracy}
\begin{center}
\begin{small}
\resizebox{\columnwidth}{!}{
\begin{tabular}{lccc}
\toprule
 & HarsanyiNet  & \makecell{Shallow \\ShapNet} & \makecell{Deep \\ShapNet\tablefootnote{The results were obtained using the codes released by the original paper~\cite{wang2021shapley}. In particular, for image datasets, each experiment was run for ten rounds with different random initialization, and the best result from the 10 runs was presented.}}\\
\midrule
\multicolumn{4}{c}{MNIST dataset}   \\
\midrule
Classification accuracy ($\uparrow$)& \textbf{99.16}   & 40.18 & 93.85   \\
Errors of Shapley values ($\downarrow$) & \textbf{1.19e-07} & 3.79e-07 &0.891\\
\midrule
\multicolumn{4}{c}{CIFAR-10 dataset}   \\
\midrule
Classification accuracy ($\uparrow$)& \textbf{89.34}   &20.48  & 73.51  \\
Errors of Shapley values ($\downarrow$) & \textbf{6.88e-08} &2.41e-07  &0.409\\
\midrule
\multicolumn{4}{c}{Census dataset}   \\
\midrule
Classification accuracy ($\uparrow$)    & 84.57 & 84.14 &  \textbf{84.72}\\
Errors of Shapley values ($\downarrow$) & \textbf{2.18e-08} & 5.12e-07  & 0.412\\
\midrule
\multicolumn{4}{c}{Yeast dataset}   \\
\midrule
Classification accuracy ($\uparrow$)& \textbf{59.91}  & 57.17 & 59.70   \\
Errors of Shapley values ($\downarrow$)& \textbf{3.36e-08} & 1.97e-07 & 0.127\\
\midrule
\multicolumn{4}{c}{TV news dataset}   \\
\midrule
Classification accuracy ($\uparrow$)& 82.20   & 79.72 & \textbf{82.46}   \\
Errors of Shapley values ($\downarrow$) &\textbf{5.69e-08} & 2.47e-07 &0.239\\
\bottomrule
\end{tabular}}
\end{small}
\end{center}
\vskip -0.2in
\end{table}
\begin{table}[t]
\caption{Error of the computed Shapley values on the Census, Yeast and TV news dataset.}
\vskip -0.1in
\label{tab:fastanddeep}
\begin{center}
\begin{small}
\resizebox{\columnwidth}{!}{
\begin{tabular}{lccc}
\toprule
\diagbox[width=10em]{Datasets}{Models} & HarsanyiNet  & DeepSHAP & FastSHAP\\
\midrule
Census & \textbf{2.18e-08} & 0.701 & 0.270\\
\midrule
Yeast & \textbf{3.36e-08} & 1.311 & 0.467 \\

\midrule
TV news &\textbf{5.69e-08} & 0.758 & 0.526\\
\bottomrule
\end{tabular}}
\end{small}
\end{center}
\vskip -0.2in
\end{table}

\begin{figure*}[h]
\centering
\includegraphics[width=0.95\linewidth,trim={3cm 4.5cm 5.2cm 10cm},clip]{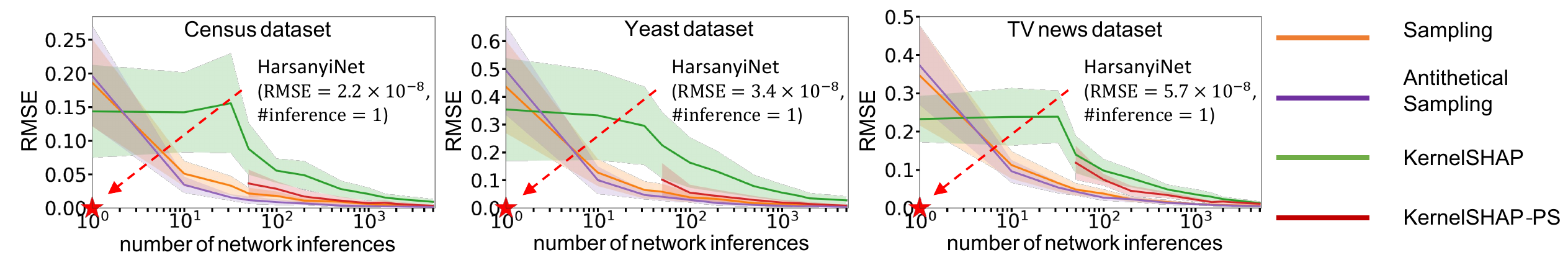}
\vskip -0.1in
\caption{Comparison of estimation errors and the computational cost (number of network inferences) required by different methods.}
\label{Fig:convergence_tabular}
\end{figure*}

\textbf{Comparing with the ShapNets.} Besides, we also compared the classification accuracy and the accuracy of Shapley values with two types of ShapNet~\cite{wang2021shapley}, namely \textit{Shallow ShapNet} and \textit{Deep ShapNet}. Input samples in the MNIST and CIFAR-10 datasets contained many more input variables. To calculate the ground-truth Shapley values through Definition~\ref{def:shapley}, we randomly sampled $n=12$ variables as input variables in the foreground of the sample $\mathbf{x}$. In this way, ground-truth Shapley values were computed by masking the selected 12 variables and keeping all the other variables as original values of these variables. Similarly, we could still use the Harsanyi-CNN and ShapNets to derive the Shapley value when we only considered $n=12$ input variables (please see Appendix~\ref{appendix:experiment12players} for details).

\begin{figure}[t!]
\centering
\vskip -0.1in
\includegraphics[width=1.0\linewidth]{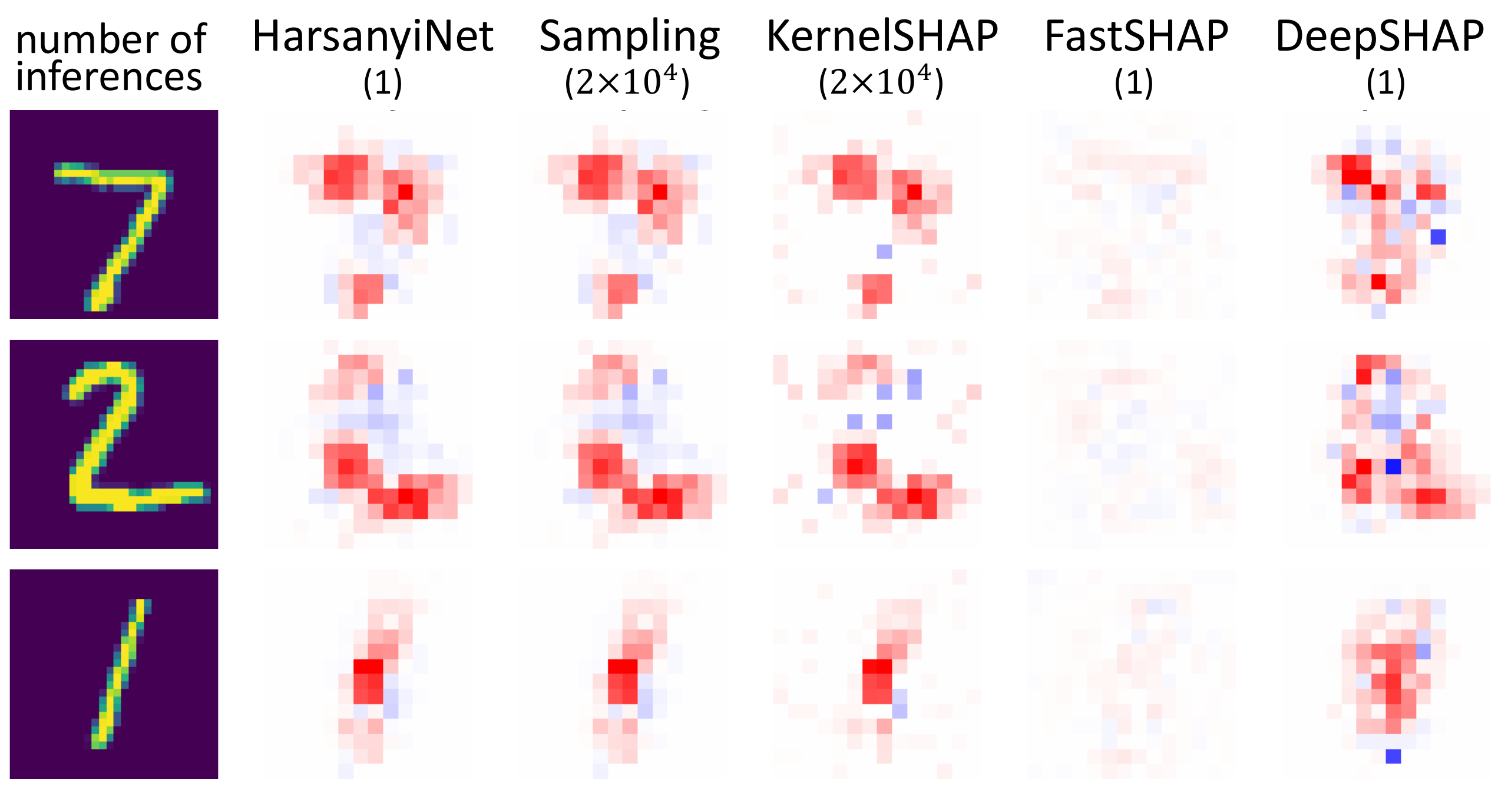}
\vskip -0.05in
\includegraphics[width=1.0\linewidth]{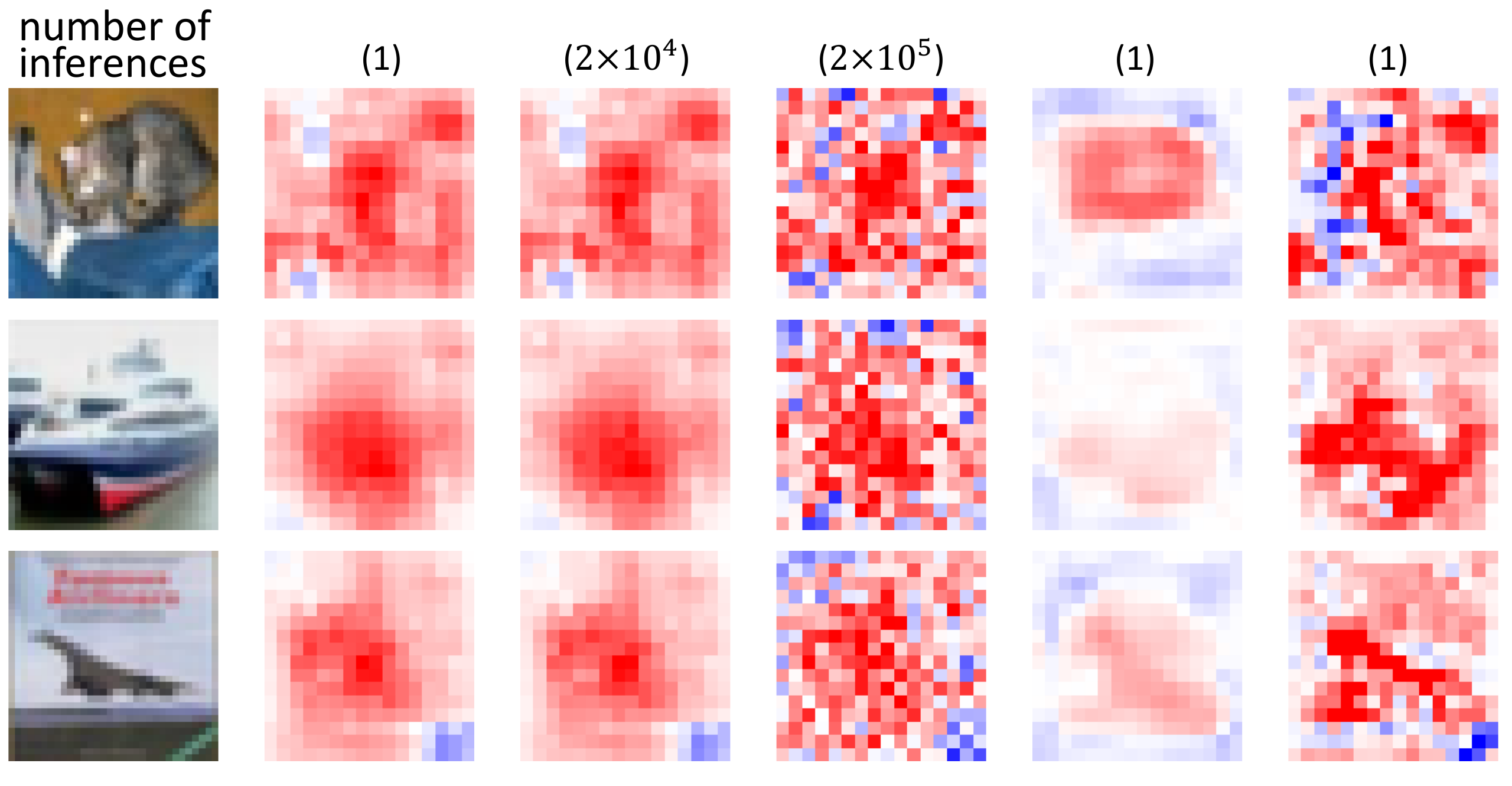}
\vskip -0.05in
\caption{Shapley values computed\footref{fn:fn1} by different methods. The number of inferences conducted for approximation is also shown.}
\vskip -0.15in
\label{Fig:cifar}
\end{figure}

Table~\ref{tab:tabularaccuracy} shows that both the HarsanyiNet and the Shallow ShapNet generated exact Shapley values with negligible errors, which were caused by unavoidable computational errors, but the HarsanyiNet had much higher classification accuracy than the Shallow ShapNet. This was because the representation capacity of the Shallow ShapNet was limited and could only encode interactions between a few input variables. On the other hand, the Deep ShapNet could not compute the exact Shapley values, although the Deep ShapNet achieved higher classification accuracy than the Shallow ShapNet. This was because the Deep ShapNet managed to encode interactions between more input variables, but the cost was that the Deep ShapNet could no longer theoretically guarantee the accuracy of the estimated Shapley values. Despite of this, the HarsanyiNet performed much better than the Deep ShapNet on more sophisticated tasks, such as image classification on the CIFAR-10 dataset.

\textbf{Visualization.} We generated attribution maps based on the Shapley values estimated by each method on the MNIST and CIFAR-10 datasets. As Figure~\ref{Fig:cifar} shows\footnote{To facilitate comparison with other methods, for the MNIST dataset, the Harsanyi-CNN was constructed with $4$ cascaded Harsanyi blocks, and each Harsanyi block had $32 \times 14 \times 14$ neurons, where $32$ is the number of channels. The hyperparameters were set to $\beta = 100$ and $\gamma = 0.05$, respectively. For the CIFAR-10 dataset, the Harsanyi-CNN was constructed with $10$ cascaded Harsanyi blocks, and each Harsanyi block had $256 \times 16 \times 16$ neurons, where $256$ is the number of channels. The hyperparameters were set to $\beta = 1000$ and $\gamma = 1$, respectively. Please see Appendix~\ref{appendix:attrimap} for more details.}, the attribution maps generated by the HarsanyiNet were almost the same as Shapley values, which were estimated by conducting inferences on 20000 sampled masked images and had converged to the true Shapley values. We also visualized the receptive fields of Harsanyi units on digit image in Figure~\ref{Fig:rf}. It verified that we could obtain the receptive field of a Harsanyi unit $z_u^{(l)}$ by merging receptive fields of its children nodes.

\begin{figure}[t!]
\centering
%\vskip 0.05in
\includegraphics[width=1.0\linewidth,trim={0cm 2cm 0cm 1cm}]{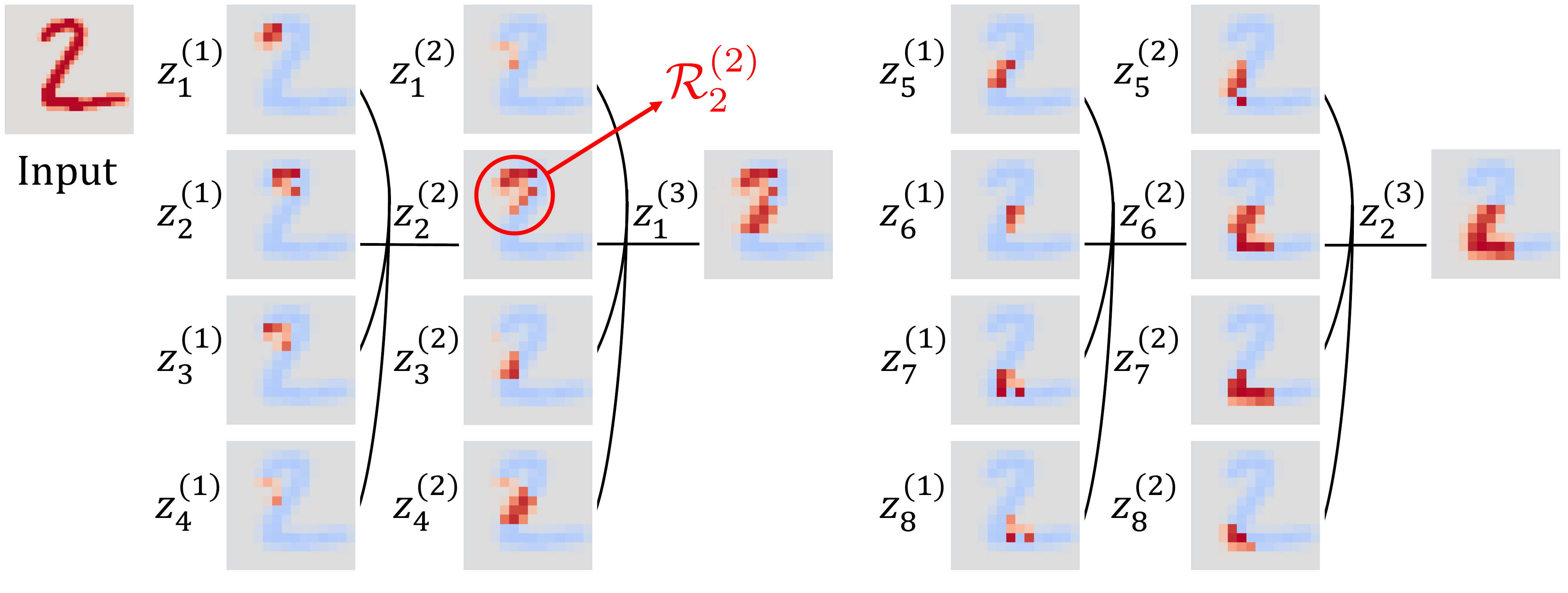}
%\vskip 0.2in
\caption{Visualization of the receptive field of Harsanyi units and the corresponding children nodes.}
\label{Fig:rf}
\vskip -0.15in
\end{figure}
\textbf{Implementation details.} The Harsanyi-MLP was constructed with 3 cascaded Harsanyi blocks, where each was formulated by following~\cref{eq:implementation1,eq:implementation2,eq:implementation3}, and each Harsanyi block had 100 neurons. The Harsanyi-CNN was constructed
with 10 cascaded Harsanyi blocks upon the feature $\mathbf{z}^{(0)}$, and each Harsanyi block had $512 \times 16\times16$ neurons, where 512 is the number of channels. The hyperparameters were set to $\beta = 10$ and $\gamma=100$ for Harsanyi-MLP trained on tabular data, and set $\beta=1000$ and $\gamma=1$ for Harsanyi-CNN trained on the image data respectively. For the Harsanyi-MLP, we randomly selected 10 neurons in the previous layer as the initial children set $\mathcal{S}^{(l)}_u$, and set the corresponding dimensions in $\boldsymbol{\tau}_u^{(l)}$ to 1. For all other neurons in the previous layer, their corresponding dimensions in $\boldsymbol{\tau}_u^{(l)}$ were initialized to $-1$. For the Harsanyi-CNN, we initialized each parameter $(\boldsymbol{\tau}_u^{(l)})_i \sim \mathcal{N}(0, 0.01^2)$, which randomly selected about half of the neurons in the previous layer to satisfy $(\boldsymbol{\tau}_u^{(l)})_i>0$ as the initial children set $\mathcal{S}^{(l)}_u$.

\textbf{Discussion on evaluation metrics for attributions.} Actually, many other metrics have been used to evaluate attribution methods, such as ROAR~\citep{hooker2019benchmark} and weakly-supervised object localization~\cite{zhou2016learning, schulz2020restricting}. As Table~\ref{tab:tabularaccuracy} and Figure~\ref{Fig:convergence_tabular} show that the HarsanyiNet generated the fully accurate Shapley values, the evaluation of the attribution generated by the HarsanyiNet should be the same as the Shapley values, and the performance of Shapley values had been sophisticatedly analyzed in previous studies~\cite{lundberg2017unified, chen2018shapley,wang2021shapley, jethani2021fastshap}. In particular, the Shapley value did not always perform the best in all evaluation metrics, although it was considered one of the most standard attribution methods and satisfied \textit{linearity}, \textit{dummy}, \textit{symmetry}, and \textit{efficiency} axioms.

\section{Conclusion}\label{sec:conclusion}
In this paper, we have proposed the HarsanyiNet that can simultaneously perform model inference and compute the exact Shapley values of input variables in a single forward propagation. We have theoretically proved and experimentally verified the accuracy of Shapley values computed by the HarsanyiNet. Only negligible errors at the level of {\small{$10^{-8}$} -- \small{$10^{-7}$}} were caused by unavoidable computational errors. Furthermore, we have demonstrated that the HarsanyiNet does not constrain the interactions between input variables, thereby exhibiting strong representation power.

\section*{Acknowledgements}
This work is partially supported by the National Nature Science Foundation of China (62276165), National Key R$\&$D Program of China (2021ZD0111602), Shanghai Natural Science Foundation (21JC1403800,21ZR1434600), National Nature Science Foundation of China (U19B2043), Shanghai Municipal Science and Technology Key Project (2021SHZDZX0102).

\bibliography{example_paper}
\bibliographystyle{icml2023}

%%%%%%%%%%%%%%%%%%%%%%%%%%%%%%%%%%%%%%%%%%%%%%%%%%%%%%%%%%%%%%%%%%%%%%%%%%%%%%%
%%%%%%%%%%%%%%%%%%%%%%%%%%%%%%%%%%%%%%%%%%%%%%%%%%%%%%%%%%%%%%%%%%%%%%%%%%%%%%%
% APPENDIX
%%%%%%%%%%%%%%%%%%%%%%%%%%%%%%%%%%%%%%%%%%%%%%%%%%%%%%%%%%%%%%%%%%%%%%%%%%%%%%%
%%%%%%%%%%%%%%%%%%%%%%%%%%%%%%%%%%%%%%%%%%%%%%%%%%%%%%%%%%%%%%%%%%%%%%%%%%%%%%%
\newpage
\appendix
\onecolumn
\setcounter{theorem}{1}
\section{The Shapley values} \label{sec:appendix_shapley}
In this section, we revisits the four axioms that the Shapley values satisfy, which ensures the Shapley values as relatively faithful attribution values. Let us consider the following cooperative game $V: 2^N \mapsto \mathbb{R}$, in which a set of $n$ players $N=\{1,2,\dots,n\}$ collaborate and win a reward $R$. Here, $V(S)$ is equivalent to $v(\mathbf{x}_S) - v(\mathbf{x}_\emptyset)$ mentioned in the paper, and we have $V(\emptyset)=0$.
%a DNN and an input sample $\mathbf{x}$ with $n$ variables $N=\{1,2,\dots,n\}$. The DNN can be understood as a game $v(\cdot)$, $v(\mathbf{x})$ can be defined as the $c$-th dimension of the network output, where $c$ usually denotes the ground-truth category of the input. 
\citet{young1985monotonic} proved that the Shapley value was the unique solution which satisfied the four axioms, including the \textit{linearity} axiom, \textit{dummy} axiom, \textit{symmetry} axiom and \textit{efficiency} axiom~\citep{weber1988probabilistic} .

(1) \textit{Linearity axiom:} If the game $V(\cdot)$ is a linear combination of two games $U(\cdot)$, $W(\cdot)$ for all $S\subseteq N$, \textit{i.e.} $V(S) = U(S) + W(S)$ and $(c\cdot V)(S) = c\cdot V(S), \forall c\in \mathbb{R}$, then the Shapley value in the game $V$ is also a linear combination of that in the games $U$ and $W$ , \textit{i.e.} $\forall i \in N, \phi^V(i) = \phi^U(i) + \phi^W(i)$ and $\phi^{c\cdot V}(i) = c\cdot\phi^V(i)$. 

(2) \textit{Dummy axiom:} A player $i$ is defined as a dummy player if $V({S\cup \{i\}}) = V(S) + V({\{i\}})$ for every $S\subseteq N\setminus\{i\}$. The dummy player $i$ satisfies $\phi(i) = V(\{i\})$, which indicates player $i$ influence the overall reward alone, without interacting/cooperating with other players in $N$.

(3) \textit{Symmetry axiom:} For two players $i$ and $j$, if $\forall S\subseteq N\setminus\{i,j\}$, $V({S\cup \{i\}}) = V({S\cup \{j\}})$, then the Shapley values of players $i$ and $j$ are equal, \textit{i.e.} $\phi(i) = \phi(j)$.

(4) \textit{Efficiency axiom:} The overall reward is equal to the sum of the Shapley value of each player, $\textit{i.e.} \sum_{i=1}^n \phi(i) = V(N)$.

\section{Proofs of Theorems}\label{appdix:proof}
\label{appendix:theorems}
In this section, we prove the theorems in the paper. 

\begin{theorem}
     Let a network output $v(\mathbf{x})\in \mathbb{R}$ be represented as $v(\mathbf{x}) = \sum\nolimits_{l=1}^{L}(\mathbf{w}_v^{(l)})^\intercal \mathbf{z}^{(l)}(\mathbf{x})$, according to~\cref{eq:efficiency}. In this way, the Harsanyi interaction between input variables in the set $S$ computed on the network output $v(\mathbf{x})$ can be represented as $I(S) = \sum\nolimits_{l=1}^{L}\sum\nolimits_{u=1}^{m^{(l)}}w_u^{(l)}J_u^{(l)}(S).$
\end{theorem}
\begin{proof}
    We have
    \begin{align*}
        v(\mathbf{x}) &= \sum\nolimits_{l=1}^{L}(\mathbf{w}_v^{(l)})^\intercal \mathbf{z}^{(l)}(\mathbf{x}) \\
        & = \sum\nolimits_{l=1}^{L}\sum\nolimits_{u=1}^{m^{(l)}}w_u^{(l)}z_u^{(l)}(\mathbf{x}). 
    \end{align*}

    According to the linearity property of the Harsanyi interactions, if $\forall S \subseteq N $, $v(\mathbf{x}_S) = u(\mathbf{x}_S) + w(\mathbf{x}_S)$ and $(cv)(\mathbf{x}_S) = c\cdot v(\mathbf{x}_S), \forall c\in \mathbb{R}$, then the Harsanyi interaction $I^v(S)$ is also a linear combination of $I^u(S)$ and $I^w(S)$, \ie, $\forall S \subseteq N $, $I^v(S) =I^u(S) + I^w(S) $ and $I^{(cv)}(S) = c\cdot I^v(S)$. Therefore, as $J^{(l)}_u(S)$ denotes the Harsanyi interaction computed on the function $z^{(l)}_u(\mathbf{x})$, we have the Harsanyi interaction computed on network output $v(\mathbf{x})$ a linear combination of $J^{(l)}_u(S)$, \ie,
    \begin{equation*}
        I(S) = \sum\nolimits_{l=1}^{L}\sum\nolimits_{u=1}^{m^{(l)}}w_u^{(l)}J_u^{(l)}(S).
    \end{equation*}
\end{proof}
\begin{theorem}[Deriving Shapley values from Harsanyi units in intermediate layers]
    The Shapley value $\phi(i)$ can be computed as
    \begin{equation*}
        \phi(i) = \sum\nolimits_{l=1}^{L}\sum\nolimits_{u=1}^{m^{(l)}} \frac{1}{|\mathcal{R}_u^{(l)}|} w_u^{(l)} z_u^{(l)}(\mathbf{x})\mathbbm{1}(\mathcal{R}_u^{(l)}\ni i).
        %\phi(i) = \!\!\!\!\!\sum_{S \subseteq N, S\ni i}\frac{1}{|S|}I(S) = \!\!\!\!\!\sum_{S \subseteq N, S\ni i}\frac{1}{|S|}\sum_uw_u J_u(S).
    \end{equation*}
\end{theorem}

\begin{proof}
    According to Theorem~\ref{theorem1} and Theorem~\ref{thm:linearity}, we have
    \begin{align*}
        \phi(i) &= \sum\nolimits_{S\subseteq N: S \ni i} \frac{1}{|S|}I(S) \\
        &=\sum\nolimits_{S\subseteq N}\frac{1}{|S|}I(S)\mathbbm{1}(S\ni i) \\
        &=\sum_{l=1}^{L}\sum_{u=1}^{m^{(l)}} \frac{1}{|\mathcal{R}_u^{(l)}|} w_u^{(l)} z_u^{(l)}(\mathbf{x})\mathbbm{1}(\mathcal{R}_u^{(l)}\ni i).
    \end{align*}
    
\end{proof}

\begin{theorem}
Based on~\cref{eq:implementation1,eq:implementation2,eq:implementation3}, the receptive field $\mathcal{R}^{(l)}_u$ of the neuron $z_u^{(l)}$ automatically satisfies the Requirement~\ref{requirement:1} and~\ref{requirement:2}. The receptive field $\mathcal{R}^{(l)}_u$ of a neuron $(l,u)$ is defined recursively by $
        \mathcal{R}^{(l)}_u := \cup_{(l',u')\in \mathcal{S}^{(l)}_{u}}\mathcal{R}^{(l')}_{u'}, \quad \st~\mathcal{R}^{(1)}_u := \mathcal{S}^{(1)}_u.$
\end{theorem}
\begin{proof}

\textbf{(1) Proof of the receptive field $\mathcal{R}^{(l)}_u$ of the neuron $z_u^{(l)}$ satisfies the Requirement~\ref{requirement:1}}.

Given two arbitrary samples $\tilde{\mathbf{x}}=\tilde{\mathbf{z}}^{(0)}$ and $\mathbf{x}=\mathbf{z}^{(0)}$, to satisfy the Requirement~\ref{requirement:1}, we will prove that if $~\forall~i\in \mathcal{R}_u^{(l)}, \tilde{\mathbf{x}}_i=\mathbf{x}_i$, then $z_u^{(l)}(\tilde{\mathbf{x}}) = z_u^{(l)}(\mathbf{x})$.

\textbf{Firstly}, for the first layer, $\forall u', \mathcal{R}_{u'}^{(1)} = \mathcal{S}^{(1)}_{u'} \subseteq \mathcal{R}_{u}^{(l)}$, we prove $ z_{u'}^{(1)}(\tilde{\mathbf{x}}) = z_{u'}^{(1)}(\mathbf{x})$. 

We get 
$g^{(1)}_{u'}(\tilde{\mathbf{x}}) =(\mathbf{A}^{(1)}_{u'})^\intercal \cdot\left(\mathbf{\Sigma}_{u'}^{(1)}\cdot \tilde{\mathbf{z}}^{(0)}\right) = (\mathbf{A}^{(1)}_{u'})^\intercal \cdot \boldsymbol{\zeta}$, 
where $\forall i\in \mathcal{R}^{(1)}_{u'} = \mathcal{S}_{u'}^{(1)} \subseteq \mathcal{R}_{u}^{(l)}, \boldsymbol{\zeta}_i = \tilde{\mathbf{z}}^{(0)}_i$, otherwise $\boldsymbol{\zeta}_i = 0$. We also get
$g^{(1)}_{u'}(\mathbf{x}) =(\mathbf{A}^{(1)}_{u'})^\intercal \cdot\left(\mathbf{\Sigma}_{u'}^{(1)}\cdot \mathbf{z}^{(0)}\right) = (\mathbf{A}^{(1)}_{u'})^\intercal \cdot \boldsymbol{\eta}$, where $\forall i\in \mathcal{R}^{(1)}_{u'}, \boldsymbol{\eta}_i = \mathbf{z}^{(0)}_i = \tilde{\mathbf{z}}^{(0)}_i$, otherwise $\boldsymbol{\eta}_i = 0$. 

Thus, $g_{u'}^{(1)}(\tilde{\mathbf{x}}) = g_{u'}^{(1)}(\mathbf{x})$ and $z_{u'}^{(1)}(\tilde{\mathbf{x}}) = z_{u'}^{(1)}(\mathbf{x})$.

\textbf{Secondly}, we prove $z_u^{(l)}(\tilde{\mathbf{x}}) = z_u^{(l)}(\mathbf{x})$ using the above conclusion.

For the second layer, $\forall u', \mathcal{R}^{(2)}_{u'} = \cup_{(1,u'')\in \mathcal{S}^{(2)}_{u'}}\mathcal{R}^{(1)}_{u''} \subseteq \mathcal{R}_{u}^{(l)}$, we can get $ z_{u'}^{(2)}(\tilde{\mathbf{x}}) = z_{u'}^{(2)}(\mathbf{x})$ easily, since its children nodes is selected from $\forall u'', \mathcal{R}_{u''}^{(1)} = \mathcal{S}^{(1)}_{u''} \subseteq \mathcal{R}_{u}^{(l)}$, and the output of which satisfies $ z_{u''}^{(1)}(\tilde{\mathbf{x}}) = z_{u''}^{(1)}(\mathbf{x})$. Similarly, we can derive $z_u^{(l)}(\tilde{\mathbf{x}}) = z_u^{(l)}(\mathbf{x})$ recursively. 

In this way, we have proved that the receptive field $\mathcal{R}^{(l)}_u$ of the neuron $z_u^{(l)}$ satisfies the Requirement~\ref{requirement:1}.

\textbf{(2) Proof of the receptive field $\mathcal{R}^{(l)}_u$ of the neuron $z_u^{(l)}$ satisfies the Requirement~\ref{requirement:2}}.

Given a sample $\mathbf{x}=\mathbf{z}^{(0)}$ and its arbitrary masked sample $\mathbf{x}_S=\mathbf{z}_S^{(0)}$, to satisfy the Requirement~\ref{requirement:2}, we will prove that $z_u^{(l)}(\mathbf{x}_S) = z_u^{(l)}(\mathbf{x})\cdot \prod_{i\in \mathcal{R}_u^{(l)}}\mathbbm{1}(i \in S)$. Specifically, we will prove that under the conditions of (1) $\forall S\supseteq \mathcal{R}_u^{(l)}$, (2) $\forall S\subsetneq \mathcal{R}_u^{(l)}$, or $\forall S, S \cup \mathcal{R}_u^{(l)}\neq S$ and $S \cup \mathcal{R}_u^{(l)}\neq \mathcal{R}_u^{(l)}$, we can get $z_u^{(l)}(\mathbf{x}_S) = z_u^{(l)}(\mathbf{x})\cdot \prod_{i\in \mathcal{R}_u^{(l)}}\mathbbm{1}(i \in S)$, respectively.

\textbf{Firstly}, we can easily get $\forall S\supseteq \mathcal{R}_u^{(l)}, z_u^{(l)}(\mathbf{x}_S) = z_u^{(l)}(\mathbf{x})\cdot \prod_{i\in \mathcal{R}_u^{(l)}}\mathbbm{1}(i \in S)$. Since $\forall i\in \mathcal{R}_u^{(l)}, (\mathbf{x}_S)_i=\mathbf{x}_i$, let us use the proven conclusion of (1) to derive $z_u^{(l)}(\mathbf{x}_S) = z_u^{(l)}(\mathbf{x}) = z_u^{(l)}(\mathbf{x})\cdot \prod_{i\in \mathcal{R}_u^{(l)}}\mathbbm{1}(i \in S)$.

\textbf{Secondly}, we prove that under the conditions of $\forall S\subsetneq \mathcal{R}_u^{(l)}$, or $\forall S, S \cup \mathcal{R}_u^{(l)}\neq S$ and $S \cup \mathcal{R}_u^{(l)}\neq \mathcal{R}_u^{(l)}$, we can get $z_u^{(l)}(\mathbf{x}_S) = z_u^{(l)}(\mathbf{x})\cdot \prod_{i\in \mathcal{R}_u^{(l)}}\mathbbm{1}(i \in S)$. 

Let $\mathbf{x}_S$ denote the sample obtained by masking variables with $\mathbf{b}$ in the set $N\setminus S$ in the sample $\mathbf{x}$, then $\mathbf{z}^{(0)}=\mathbf{x}-\mathbf{b} \in \mathbb{R}^n$. In both settings, there exists at least a variable $j$ that belongs to $\mathcal{R}_u^{(l)}$ but not to $S$, \ie, $\exists j\in \mathcal{R}_u^{(l)}, j\notin S$, we have $(\mathbf{x}_S)_j = b$,  $(\mathbf{z}^{(0)}_S)_j =  0$ and $\prod_{i\in \mathcal{R}_u^{(l)}}\mathbbm{1}(i \in S)=0$. 

For the first layer, there exists at least a neuron $(1, u')$ which satisfies $j\in \mathcal{R}_{u'}^{(1)} = \mathcal{S}^{(1)}_{u'} \subseteq \mathcal{R}_{u}^{(l)}$. Then $\forall u', h^{(1)}_{u'}(\mathbf{x}_S) = g^{(1)}_{u'}(\mathbf{x}_S) \cdot\prod_{(0,u'')\in\mathcal{S}_{u'}^{(1)}}\mathbbm{1}(z^{(0)}_{u''}(\mathbf{x}_S)\neq 0) = g^{(1)}_{u'}(\mathbf{x}_S) \cdot \mathbbm{1}((\mathbf{z}^{(0)}_S)_j \neq 0) =0$ and $z^{(1)}_{u'}(\mathbf{x}_S) = 0$. Since $j \in \mathcal{R}^{(l)}_u = \cup_{(l',u'')\in \mathcal{S}^{(l)}_{u}}\mathcal{R}^{(l')}_{u''}$, there exists at least a neuron $(1, u')$ will affect the neuron $(l, u)$ recursively, \ie, $h^{(l)}_{u}(\mathbf{x}_S) = 0$ and $z^{(l)}_{u}(\mathbf{x}_S) = 0$. Thus, $z_u^{(l)}(\mathbf{x}_S) = z_u^{(l)}(\mathbf{x})\cdot \prod_{i\in \mathcal{R}_u^{(l)}}\mathbbm{1}(i \in S) = 0$.

In this way, we have proved that the receptive field $\mathcal{R}^{(l)}_u$ of the neuron $z_u^{(l)}$ satisfies the Requirement~\ref{requirement:2}.

\end{proof}

\section{Proof of Lemma~\ref{thm:J(S)}}
\label{appendix:lemma}
\setcounter{lemma}{0}
\begin{lemma}[Harsanyi interaction of a Harsanyi unit]
    Let us consider the output of a Harsanyi unit $z_u^{(l)}(\mathbf{x})$ as the reward. Then, let $J_u^{(l)}(S)$ denote the Harsanyi interaction \wrt~the function $z_u^{(l)}(\mathbf{x})$. Then, we have $J_u^{(l)}(\mathcal{R}_u^{(l)})=z_u^{(l)}(\mathbf{x})$, and $\forall S\neq \mathcal{R}_u^{(l)}, J_u^{(l)}(S) =0$, according to Requirements~\ref{requirement:1} and~\ref{requirement:2}.
\end{lemma}
\begin{proof}
    According to Definition~\ref{def:harsanyi}, \ie, $I(S) = v(S) - \sum_{L\subsetneq S}I(L)$ subject to $I(\emptyset) \coloneqq 0$, and Requirements~\ref{requirement:2}, \ie, $z_u^{(l)}(\mathbf{x}_S) = z_u^{(l)}(\mathbf{x})\cdot \prod_{i\in \mathcal{R}_u^{(l)}}\mathbbm{1}(i \in S)$, the Harsanyi interaction of a Harsanyi unit can be written as, 
    \begin{align*}
        J_u^{(l)}(S) &=z_u^{(l)}(\mathbf{x}_S)-\sum_{L\subsetneq S}J_u^{(l)}(L) \\ 
        &=z_u^{(l)}(\mathbf{x})\cdot\prod_{i\in \mathcal{R}_u^{(l)}}\mathbbm{1}(i \in S) -\sum_{L\subsetneq S}J_u^{(l)}(L)  
    \end{align*}

\textbf{(1) Proof of $J_u^{(l)}(\mathcal{R}_u^{(l)})=z_u^{(l)}(\mathbf{x})$}.

\textbf{Firstly}, let us use the inductive method to prove $\forall L\subsetneq \mathcal{R}_u^{(l)}, J_u^{(l)}(L)=0$.

If $|L| = 1, \forall L' \subseteq L \subsetneq \mathcal{R}_u^{(l)}$, we have $\prod_{i\in \mathcal{R}_u^{(l)}}\mathbbm{1}(i \in L')=0$, then we get $J^{(l)}_u(L') = z_u^{(l)}(\mathbf{x}) \cdot \prod_{i\in \mathcal{R}_u^{(l)}}\mathbbm{1}(i \in L') = 0$.

Assume that if $|L|=k$, $\forall L' \subseteq L \subsetneq \mathcal{R}_u^{(l)}$, we have $J^{(l)}_u(L')=0$.

Then if $|L| = k+1$, $\forall L' \subseteq L \subsetneq \mathcal{R}_u^{(l)}$, we have $\prod_{i\in \mathcal{R}_u^{(l)}}\mathbbm{1}(i \in L)=0$ and $\forall L'\subsetneq L, J_u^{(l)}(L')=0$. Thus, we get $J^{(l)}_u(L)=z_u^{(l)}(\mathbf{x})\cdot\prod_{i\in \mathcal{R}_u^{(l)}}\mathbbm{1}(i \in L) -\sum_{L'\subsetneq L}J_u^{(l)}(L')$ = 0.

In this way, we have proved that $\forall 1 \leq |L| < |\mathcal{R}_u^{(l)}|, \forall L\subsetneq \mathcal{R}_u^{(l)}, J_u^{(l)}(L)=0$.

\textbf{Secondly}, let us use the proven conclusion $\forall L\subsetneq \mathcal{R}_u^{(l)}, J_u^{(l)}(L)=0$ to derive $J_u^{(l)}(\mathcal{R}_u^{(l)})=z_u^{(l)}(\mathbf{x})$. 

Since $\prod_{i\in \mathcal{R}_u^{(l)}}\mathbbm{1}(i \in \mathcal{R}_u^{(l)})=1$, we get $J_u^{(l)}(\mathcal{R}_u^{(l)})= z_u^{(l)}(\mathbf{x})\cdot\prod_{i\in \mathcal{R}_u^{(l)}}\mathbbm{1}(i \in \mathcal{R}_u^{(l)}) -\sum_{L\subsetneq \mathcal{R}_u^{(l)}}J_u^{(l)}(L) = z_u^{(l)}(\mathbf{x})$.

In this way, we have proved that $J_u^{(l)}(\mathcal{R}_u^{(l)})=z_u^{(l)}(\mathbf{x})$. 

\textbf{(2) Proof of $\forall S\neq \mathcal{R}_u^{(l)}, J_u^{(l)}(S) =0$}.

To prove $\forall S\neq \mathcal{R}_u^{(l)}, J_u^{(l)}(S) =0$, we will prove that under the conditions of (1) $\forall S\subsetneq \mathcal{R}_u^{(l)}$, (2) $\forall S\supsetneq \mathcal{R}_u^{(l)}$, and (3) $\forall S, S \cup \mathcal{R}_u^{(l)}\neq S$ and $S \cup \mathcal{R}_u^{(l)}\neq \mathcal{R}_u^{(l)}$, we can get $J_u^{(l)}(S) =0$, respectively.

\textbf{Firstly}, we have proved that $\forall S\subsetneq \mathcal{R}_u^{(l)}, J_u^{(l)}(S)=0$.

\textbf{Secondly}, let us use the inductive method to prove $\forall S\supsetneq \mathcal{R}_u^{(l)}, J_u^{(l)}(S)=0$.

In this setting, $\prod_{i\in \mathcal{R}_u^{(l)}}\mathbbm{1}(i \in S)=1$. If $|S| = |\mathcal{R}_u^{(l)}|+1, \forall S\supsetneq \mathcal{R}_u^{(l)}$, we have $J^{(l)}_u(S) = z_u^{(l)}(\mathbf{x}) \cdot \prod_{i\in \mathcal{R}_u^{(l)}}\mathbbm{1}(i \in S) - \sum_{L\subsetneq S}J_u^{(l)}(L) = z_u^{(l)}(\mathbf{x}) - [J_u^{(l)}(\mathcal{R}_u^{(l)})+\sum_{L\subsetneq S, L\neq \mathcal{R}_u^{(l)}}J_u^{(l)}(L)]= 0$. (Similarly, $\forall L\subsetneq S$ and $ L\neq \mathcal{R}_u^{(l)}, J_u^{(l)}(L) = 0$ can be proved by the inductive method.)

Assume that if $|S| = |\mathcal{R}_u^{(l)}|+k, \forall S\supsetneq \mathcal{R}_u^{(l)}$, we have $J^{(l)}_u(S) = 0$.

Then if $|S| = |\mathcal{R}_u^{(l)}|+ (k+1), \forall S\supsetneq \mathcal{R}_u^{(l)}$, we have $J^{(l)}_u(S) = 0$.

\textbf{Thirdly}, let us use the inductive method to prove $\forall S, S \cup \mathcal{R}_u^{(l)}\neq S$ and $S \cup \mathcal{R}_u^{(l)}\neq \mathcal{R}_u^{(l)}, J_u^{(l)}(S)=0$.

In this setting, $\prod_{i\in \mathcal{R}_u^{(l)}}\mathbbm{1}(i \in S)=0$. Then we have $J^{(l)}_u(S) = z_u^{(l)}(\mathbf{x}) \cdot \prod_{i\in \mathcal{R}_u^{(l)}}\mathbbm{1}(i \in S) - \sum_{L\subsetneq S}J_u^{(l)}(L) = 0$. (Similarly, $\forall L \subsetneq S, J_u^{(l)}(S)=0$ can be proved by the inductive method.)

In this way, we have proved that $\forall S\neq \mathcal{R}_u^{(l)}, J_u^{(l)}(S) =0$.
    
\end{proof}
\section{Discussion on~\cref{eq:implementation1,eq:implementation2,eq:implementation3}}
\label{appendix: discussion}
Section~\ref{sec:construct} introduced that the neural activation $z_u^{(l)}(\mathbf{x})$ of the neuron $(l,u)$ in a Harsanyi block was computed by applying a linear operation (\cref{eq:implementation1}), an AND operation (\cref{eq:implementation2}), and a ReLU operation (\cref{eq:implementation3}). We provide further discussions on the above three operations as follows.

Unlike a linear layer in a traditional DNN,~\cref{eq:implementation1} shows that among neurons in all previous $(l-1)$ blocks, only outputs of the children nodes $\mathbf{\Sigma}_u^{(l)} \cdot \mathbbm{z}^{(l-1)}$ can affect the output of the neuron $(l, u)$.~\cref{eq:implementation2} denotes that if all children nodes in $\mathcal{S}_u^{(l)}$ are activated, then the activation score $g^{(l)}_u(\mathbf{x})$ can pass through the AND operation, \ie, $h_u^{(l)}(\mathbf{x})=g_u^{(l)}(\mathbf{x})$. Otherwise, if any children node is not activated, \ie,  $\exists (l', u')\in\mathcal{S}_u^{(l)}$, \eg, $z^{(l')}_{u'}(\mathbf{x})=0$, then we have $h_u^{(l)}(\mathbf{x})=0$.

\section{Proofs and implementation details for Harsanyi-CNN}
\label{appendix:Harsanyi-cnn}
\textbf{Harsanyi-CNN architecture.} Figure~\ref{Fig:childrenset} illustrates the architecture of the Harsanyi-CNN. As introduced in Section~\ref{sec:experimentsetting}, we first applied a convolutional layer, max-pooling layer and ReLU layer to obtain the feature $\mathbf{z}^{(0)}$. Then, we built cascaded Harsanyi blocks on $\mathbf{z}^{(0)}$. Similar to the traditional CNN, each neuron $(l, u=(c,w,h))$ in the convolutional layer of each HarsanyiBlock corresponds to a subtensor $\textbf{T}^{(l)}_u \in \mathbb{R}^{C\times K \times K}$ \wrt~the previous layer, where $C$ is the number of channels in the previous layer and $K\times K$ is the 2D kernel size. The neurons which share the same location but on different channels $(l, u=(:,w,h))$ correspond to the same subtensor $\textbf{T}^{(l)}_u$. The children set $\mathcal{S}_u^{(l)}$ of each neuron $(l,u)$ were selected from the subtensor $\textbf{T}^{(l)}_u$. Moreover, neurons on the same location but on different channels $(l', u'=(:,w',h'))$ belong to the children set $\mathcal{S}_u^{(l)}$ simultaneously. The output of the HarsanyiBlock was constructed following~\cref{eq:implementation1,eq:implementation2,eq:implementation3}. Finally, each dimension of the network output $v(\mathbf{x})$ is constructed as the weighted sum of the output of each HarsanyiBlock using linear transformations and the skip-connection. 
%As mentioned in the Section~\ref{sec:experiment}, we considered neurons in the same location on different channels as a single player. In other words, a single neuron could be considered to have all channels in an intermediate layer. In this way, we had to specify the number of cascaded Harsanyi blocks, the number of neurons per layer, and the number of channels per neuron. We use $H$, $W$ and $C$ to denote the height, width and number of channels of the features in an intermediate layer, respectively. For convenience, we kept the feature size of each intermediate layer fixed, \ie, we set the number of neurons per intermediate layer $H\times W$ and the number of channels per neuron $C$ to be the constant, respectively.

\begin{figure*}[h]
\vskip -0.1in
\centering
\includegraphics[width=0.8\linewidth,trim={0cm 0cm 4cm 1cm},clip]{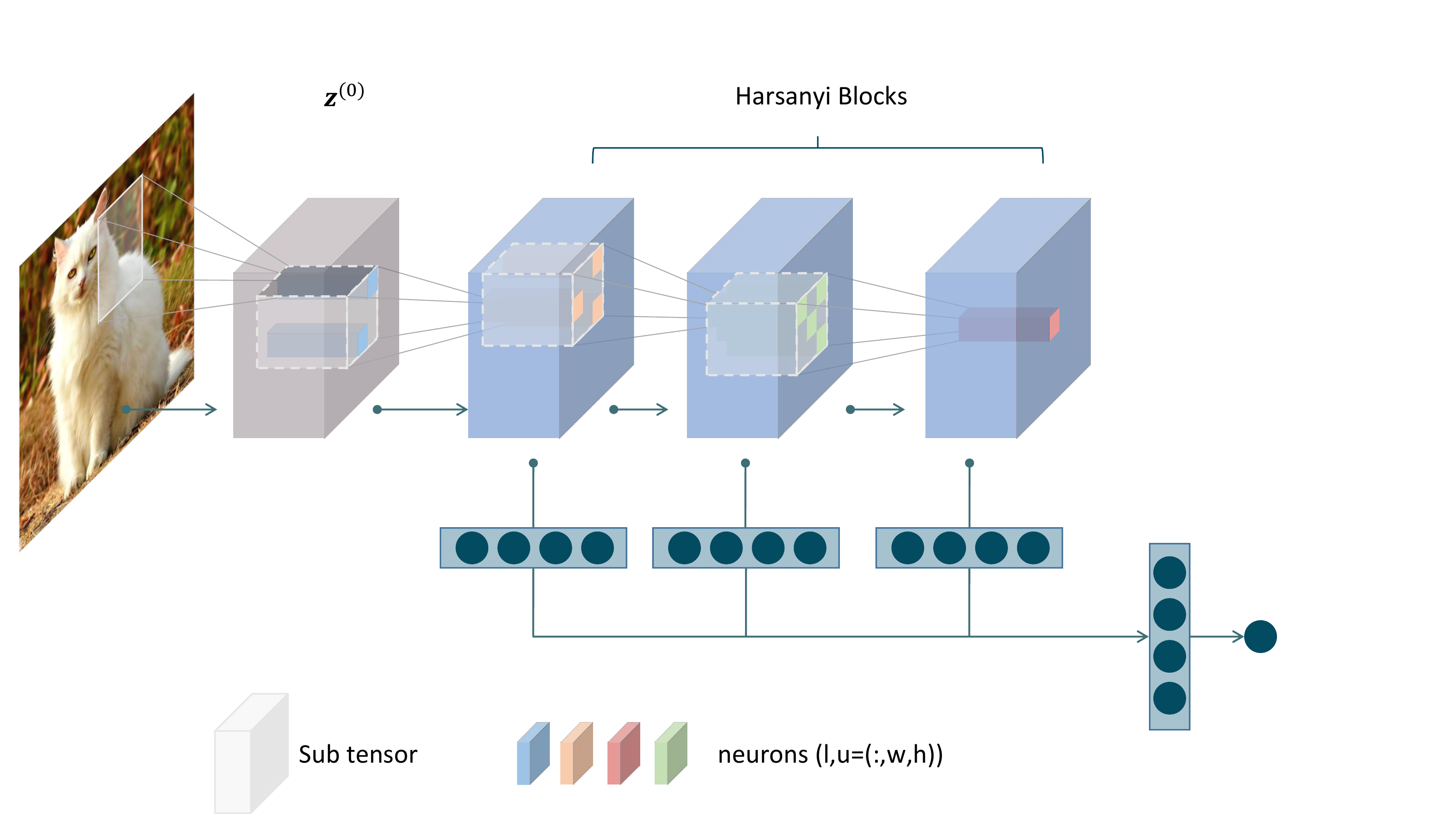}
\caption{Schematic diagram of the Harsanyi-CNN architecture.}
\label{Fig:childrenset}
\vskip -0.1in
\end{figure*}

\textbf{Proof of the conclusion in Setting 2 that}
\textit{based on the design of letting all neurons $(l,u=(1,h,w)),\dots,(l,u=(C,h,w))$ share the same parameter $\boldsymbol{\tau}_u^{(l)}$, all Harsanyi units $(l,u=(c,h,w))$ in the same location $(h,w)$ on different channels $(c=1,\dots,C)$ had the same receptive field $\mathcal{R}^{(l)}_{u=(:,h,w)}$ and contributed to the same Harsanyi interaction $I(S=\mathcal{R}^{(l)}_{u=(:,h,w)})$}.
\begin{proof}
    According to the implementation $\forall (l,u), \boldsymbol{\tau}_{u=(1,h,w)}^{(l)} = \boldsymbol{\tau}_{u=(2,h,w)}^{(l)} = \dots = \boldsymbol{\tau}_{u=(C,h,w)}^{(l)} \in \mathbb{R}^{CK^2}$, we will prove that $\forall (l,u), \mathcal{R}^{(l)}_{u=(1,h,w)} = \mathcal{R}^{(l)}_{u=(2,h,w)} = \dots = \mathcal{R}^{(l)}_{u=(C,h,w)}$.
    
    Since $\forall (l,u), (\mathbf{\Sigma}_u^{(l)})_{i,i}\!\!=\!\!\mathbbm{1}((\boldsymbol{\tau}_u^{(l)})_i>0)$, then for arbitrary binary diagonal matrix $\mathbf{\Sigma}_u^{(l)}$, we have $\forall (l,u), \mathbf{\Sigma}_{u=(1,h,w)}^{(l)}= \mathbf{\Sigma}_{u=(2,h,w)}^{(l)} = \dots = \mathbf{\Sigma}_{u=(C,h,w)}^{(l)}$. The children set $\mathcal{S}_u^{(l)}$ is implemented by  $\mathbf{\Sigma}_u^{(l)}$, then we have $\forall (l,u), \mathcal{S}_{u=(1,h,w)}^{(l)} = \mathcal{S}_{u=(2,h,w)}^{(l)} = \dots = \mathcal{S}_{u=(C,h,w)}^{(l)}$. According to Equation~(\ref{eq:rf}), we derive $\mathcal{R}_{u}^{(l)}$ from $\mathcal{S}_{u}^{(l)}$ recursively, then we have $\forall (l,u), \mathcal{R}_{u=(1,h,w)}^{(l)} = \mathcal{R}_{u=(2,h,w)}^{(l)} = \dots = \mathcal{R}_{u=(C,h,w)}^{(l)}$. In this way, the Harsanyi units $(l,u=(c,h,w))$ in the same location $(h,w)$ on different channels $(c=1,\dots,C)$ had the same receptive field $\mathcal{R}^{(l)}_{u=(:,h,w)}$.

    Next, we will show that considering $C$ channels as $C$ Harsanyi units, and considering $C$ channels together as a single Harsanyi unit, their Harsanyi interactions are equal in both cases.
    
    Considering $C$ channels as $C$ Harsanyi units, we have totally $m^{(l)} = H\times W\times C$ Harsanyi units in the $l$-th layer. We have $I(S=\mathcal{R}_{u=(1,h,w)}^{(l)}) = I(S=\mathcal{R}_{u=(2,h,w)}^{(l)}) = \dots = I(S=\mathcal{R}_{u=(C,h,w)}^{(l)})$, which is abbreviated to $I(S=\mathcal{R}_u^{(l)})$. According to Theorem~\ref{thm:linearity} and Lemma~\ref{thm:J(S)}, we have $$I(S=\mathcal{R}_u^{(l)}) = \sum\nolimits_{l=1}^{L}\sum\nolimits_{u=1}^{m^{(l)}}w_{v,u}^{(l)}J_u^{(l)}(S=\mathcal{R}_u^{(l)}) = \sum\nolimits_{l=1}^{L}\sum\nolimits_{u=1}^{H\times W\times C}w_{v,u}^{(l)}z_u^{(l)}(\mathbf{x})$$ 
    where $w_{v,u}^{(l)}\in \mathbb{R}$ and $z_u^{(l)}(\mathbf{x})\in \mathbb{R}$.  Based on~\cref{eq:implementation1,eq:implementation2,eq:implementation3},
    note that  $\forall c\in \{1,2,\dots,C\}, (l,u=(c,h,w))$ share the same children nodes $\mathcal{S}_{u=(1,h,w)}^{(l)} = \mathcal{S}_{u=(2,h,w)}^{(l)} = \dots = \mathcal{S}_{u=(C,h,w)}^{(l)}$, then $h_{u=(c,h,w)}^{(l)}(\mathbf{x})$ at the same location on different channels is activated or deactivated at the same time, due to the AND operation on the child nodes.
    Besides, $z_u^{(l)}(\mathbf{x})$ is determined by the linear combination of the child nodes $g_u^{(l)}(\mathbf{x}) = (\mathbf{A}^{(l)}_u)^\intercal\!\!\cdot\!\!\left(\mathbf{\Sigma}_u^{(l)} \!\cdot \!\mathbbm{z}^{(l-1)}\right)$, where $\mathbf{A}^{(l)}_u\in \mathbb{R}^{CK^2}$ is the parameter of a convolution kernel (A total of $C$ convolution kernels, denoted as $\mathbf{B}^{(l)}_u\in \mathbb{R}^{(CK^2)\times C}$, can derive $C$ harsanyi units at the same position on different channels), $ \mathbf{\Sigma}_{u}^{(l)}\in \mathbb{R}^{(CK^2)\times (CK^2)}$ denotes the selected children nodes and $\mathbbm{z}^{(l-1)}\in \mathbb{R}^{CK^2}$ denotes the feature maps of the $(l-1)$-th layer within the coverage of the convolution kernel. 
    
    Considering $C$ channels together as a single Harsanyi unit, we have totally $m^{(l)} = H\times W$ Harsanyi units in the $l$-th layer. We use $I(S=\mathcal{R}_{u=(:,h,w)}^{(l)})$ to denote this case. According to Theorem~\ref{thm:linearity} and Lemma~\ref{thm:J(S)}, we have 
    $$I(S=\mathcal{R}_{u=(:,h,w)}^{(l)}) = \sum\nolimits_{l=1}^{L}\sum\nolimits_{u=1}^{m^{(l)}}w_{v,u}^{(l)}J_u^{(l)}(S=\mathcal{R}_u^{(l)}) = \sum\nolimits_{l=1}^{L}\sum\nolimits_{u=1}^{H\times W}
    (\mathbf{w}_{v,u}^{(l)})^\intercal \mathbf{z}_u^{(l)}(\mathbf{x})$$ 
    where $\mathbf{w}_{v,u}^{(l)}\in \mathbb{R}^C$ and $\mathbf{z}_u^{(l)}(\mathbf{x})\in \mathbb{R}^C$. Based on~\cref{eq:implementation1,eq:implementation2,eq:implementation3}, note that $\forall c\in \{1,2,\dots,C\}, \mathcal{S}_{u=(:,h,w)}^{(l)} = \mathcal{S}_{u=(c,h,w)}^{(l)}$, then the single $C$-dimensional Harsanyi unit has the same activation state as $C$ Hassani units in above case. Besides, $\mathbf{z}_u^{(l)}(\mathbf{x})$ is determined by the linear combination of the child nodes $\mathbf{g}_u^{(l)}(\mathbf{x}) = (\mathbf{B}^{(l)}_u)^\intercal\cdot\left(\mathbf{\Sigma}_u^{(l)} \!\cdot \!\mathbbm{z}^{(l-1)}\right) \in  \mathbb{R}^C$, where $\mathbf{B}^{(l)}_u\in \mathbb{R}^{(CK^2)\times C}$ is the parameters of a total of $C$ convolution kernels, $\mathbf{\Sigma}_{u}^{(l)}\in \mathbb{R}^{(CK^2)\times (CK^2)}$ denotes the selected children nodes and $\mathbbm{z}^{(l-1)}\in \mathbb{R}^{CK^2}$ denotes the feature maps of the $(l-1)$-th layer within the coverage of the convolution kernel.

    In this way, we proved that the Harsanyi units $(l,u=(c,h,w))$ in the same location $(h,w)$ on different channels $(c=1,\dots,C)$ had the same receptive field $\mathcal{R}^{(l)}_{u=(:,h,w)}$ and contributed to the same Harsanyi interaction $I(S=\mathcal{R}^{(l)}_{u=(:,h,w)})$.
    
\end{proof}

\section{More experiment results and details}

\subsection{Experiment results of more challenging datasets on the HarsanyiNets}
To further explore the classification performance of the HarsanyiNets, we conducted experiments on more challenging datasets, including the Oxford Flowers-102~\citep{maria-elena2008automated} and COVIDx dataset~\citep{wang2020COVID-Net}. To compare the classification accuracy of the HarsanyiNet with a traditional DNN, we used ResNet-50~\citep{he2016deep} and COVID-Net CXR-2~\citep{pavlova2022covid} as baseline models and reported the results in Table~\ref{tab:imageperformance}. Specifically, we used the intermediate-layer features with the size of $512\times 14 \times 14$ from the pre-trained VGG-16 model~\citep{simonyan2015very} as $\mathbf{z}^{(0)}$, and then trained the HarsanyiNet upon $\mathbf{z}^{(0)}$ with the same hyperparameters as described in Section~\ref{sec:experiment}.

\begin{table}[t]
\caption{Classification accuracy (\%) of the HarsanyiNet and baseline models on more challenging datasets}
\label{tab:imageperformance}
\begin{center}
\begin{small}
\begin{tabular}{lccc}
\toprule
Dataset &  HarsanyiNet  &  baseline models  \\
\midrule
Oxford Flowers-102  & 95.48  & 97.9 (ResNet50~\citep{wightman2021resnet})   \\
COVIDx & 96.75 & 96.3 (COVID-Net CXR-2~\citep{pavlova2022covid}) \\
\bottomrule
\end{tabular}
\end{small}
\end{center}
\end{table}

As shown in Table~\ref{tab:imageperformance}, the classification accuracy of the HarsanyiNet on the Oxford Flowers-102 is slightly lower than ResNet-50. However, on medical dataset COVIDx, the classification accuracy of the HarsanyiNet is slightly higher than COVID-Net CXR-2. Despite this relatively small sacrifice in classification accuracy on certain datasets, the HarsanyiNet computed the exact Shapley values in a single forward propogation, which was its main advantage over other neural networks.

\subsection{Experiment results for verifying the accuracy of the Shapley values on the HarsanyiNets}

To further verify the accuracy of the Shapley values on high-dimensional image datasets, we compared the Shapley values calculated by HarsanyiNet with those estimated by the sampling method. Specifically, we ran the sampling algorithm with 5000 and 10000 iterations on the MNIST dataset and the CIFAR-10 dataset, respectively. Table~\ref{tab:imageaccuracy} shows the root mean square error (RMSE) between the Shapley values calculated by HarsanyiNet and the Shapley values estimated by the sampling algorithm. The estimation errors between both methods are quite small. Nevertheless, we need to emphasize that the sampling algorithm was more accurate when the sampling number was large, there was still a non-negligible error between the the estimated Shapely values and the ground-truth Shapley values.

\begin{table}[t]
\vskip -0.1in
\caption{Error between the Shapley values computed by the HarsanyiNet and the Shapley values estimated by the sampling method}
\label{tab:imageaccuracy}
\begin{center}
\begin{small}
\begin{tabular}{lccc}
\toprule
Dataset & Errors of Shapley values (5000 iterations)  &  Errors of Shapley values (10000 iterations)   \\
\midrule
MNIST  & 0.017  & 0.012   \\
CIFAR-10 & 0.007 & 0.004 \\
\bottomrule
\end{tabular}
\end{small}
\end{center}
\vskip -0.1in
\end{table}

\subsection{Experiment results of the training cost of the HarsanyiNets}
To further explore the training cost of the HarsanyiNets, we conducted experiments on the Census, MNIST, and CIFAR-10 datasets to evaluate the training cost of HarsanyiNets and traditional DNNs with comparable sizes. Table~\ref{tab:trainingcost} shows that the computational cost of training the HarsanyiNet is higher than training a comparable DNN, and the computational cost of the HarsanyiNet is about twice the cost of a traditional DNN with the same number of parameters.

\begin{table}[t]
\caption{Training cost per epoch (s) of the HarsanyiNet and the comparable DNN}
\label{tab:trainingcost}
\begin{center}
\begin{small}
\begin{tabular}{lccc}
\toprule
Dataset &  HarsanyiNet  &  Comparable DNN  \\
\midrule
Census  & 5.0  & 1.9   \\
MNIST & 243.3 & 127.0      \\
CIFAR-10 & 205.0 & 106.7   \\
\bottomrule
\end{tabular}
\end{small}
\end{center}
\end{table}

\subsection{Experiment results of the robustness of the HarsanyiNets}
We conducted more experiments to analyze the robustness of the HarsanyiNet. Specifically, we estimate the adversarial robustness of the classification performance and the adversarial robustness of the estimated Shapley values~\citep{jia2019towards}.

To estimate the adversarial robustness of the classification performance on HarsanyiNet, we conducted experiments on the CIFAR-10 dataset to evaluate the model robustness by examining its classification accuracy on the test set of adversarial examples. To generate adversarial examples, we used the FGSM attack~\citep{goodfellow2015explaining}, a gradient-based method, with a maximum perturbation of 8/255~\citep{madry2018towards}.
Table~\ref{tab:robustness} shows that the classification accuracy of the adversarial examples of HarsanyiNet is slightly higher than that of ResNet-18~\citep{he2016deep}.

To estimate the adversarial robustness of the estimated Shapley values on HarsanyiNet, we assessed the robustness of its estimated Shapley values by computing the $\ell_2$ norm of the difference in Shapley values between the adversarial and natural examples, \ie, $||\boldsymbol{\phi}^{nat} - \boldsymbol{\phi}^{adv}||_{\ell_2}$, where $\boldsymbol{\phi}^{nat}$ denotes the Shapley values of natural examples, and $\boldsymbol{\phi}^{adv}$ denotes the Shapley values of adversarial examples. To calculate the Shapley values of the ResNet-18 model, we estimate Shapley values using the sampling algorithm with 1000 iterations. Table~\ref{tab:robustness} shows that the adversarial robustness of the estimated Shapley values on HarsanyiNet (estimated by the $\ell_2$ norm of the difference of Shapley values, the lower the better) is slightly higher than that of ResNet-18.

Both experiments indicate that HarsanyiNet has a robustness close to, or even slightly higher than, that of the traditional model.

\begin{table}[t]
\vskip -0.2in
\caption{Model robustness and Shapley value robustness}
\label{tab:robustness}
\begin{center}
\begin{small}
\begin{tabular}{lccc}
\toprule
Model & Classification accuracy of adversarial examples (\%)  &  $\ell_2$ norm of the Shapley value difference   \\
\midrule
HarsanyiNet  & 13.83\%  & 1.44   \\
ResNet-18 & 8.21\% & 1.50 \\
\bottomrule
\end{tabular}
\end{small}
\end{center}
\vskip -0.1in
\end{table}

\subsection{More results of the estimated Shapley values on the HarsanyiNets}
We conducted more experiments to show the explanations produced by our HarsanyiNets. Specifically, we trained the Harsanyi-MLP on tabular datasets and the Harsanyi-CNN on image datasets.

For the tabular datasets including the Census, Yeast and TV news datasets, we compared the estimated Shapley values for each method in Figure~\ref{Fig:census_example}, Figure~\ref{Fig:yeast_example}, and Figure~\ref{Fig:commercial_example}, respectively. It can be seen that the Shapley values calculated by our HarsanyiNet were exactly the same as the ground-truth Shapley values calculated by Definition~\ref{def:shapley}, while the approximation methods, including the sampling method~\citep{castro2009polynomial}, antithetical sampling~\citep{mitchell2022sampling}, KernelSHAP~\citep{lundberg2017unified}, and KernelSHAP with paired sampling (KernelSHAP-PS)~\citep{covert2021improving}, needed thousands of network inferences to compute the relatively accurate Shapley values. 

For the image datasets including the MNIST and CIFAR-10 datasets, we generated more attribution maps on different categories in Figure~\ref{Fig:mnist_classes} and Figure~\ref{Fig:cifar_classes}, respectively.

\begin{figure}[t!]
\centering
\includegraphics[width=0.9\linewidth,trim={1cm 2.5cm 3cm 0cm}]{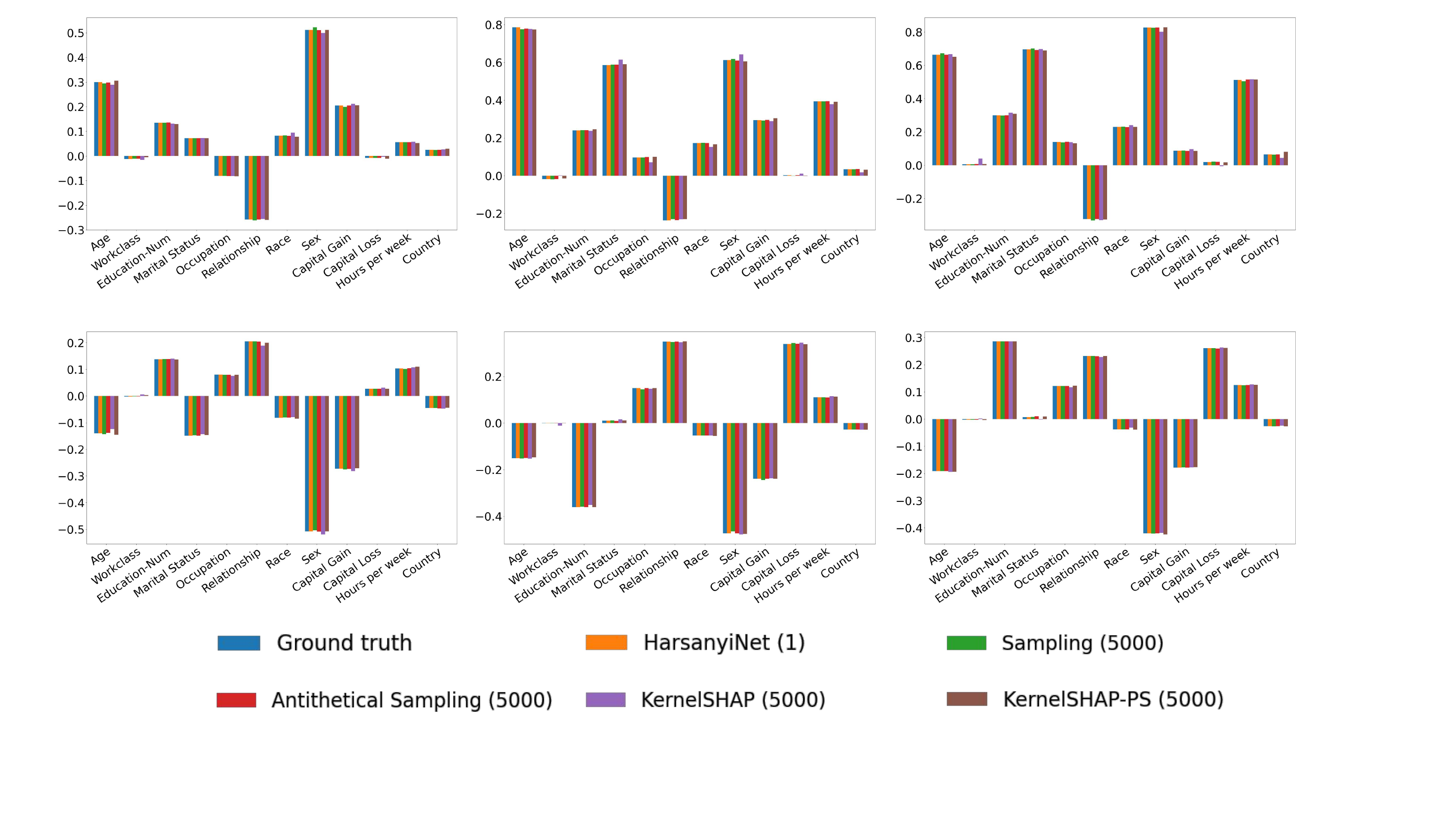}
\caption{Shapley values computed by different methods on the Census dataset. The number of inferences conducted for each method is indicated in the brackets. The samples in the first row are from category `$\leq$50K' and the samples in the second row are from the category `$>$50K'.}
\label{Fig:census_example}
\end{figure}

\begin{figure}[t!]
\centering
\includegraphics[width=1.0\linewidth,trim={0cm 0cm 0cm 0cm}]{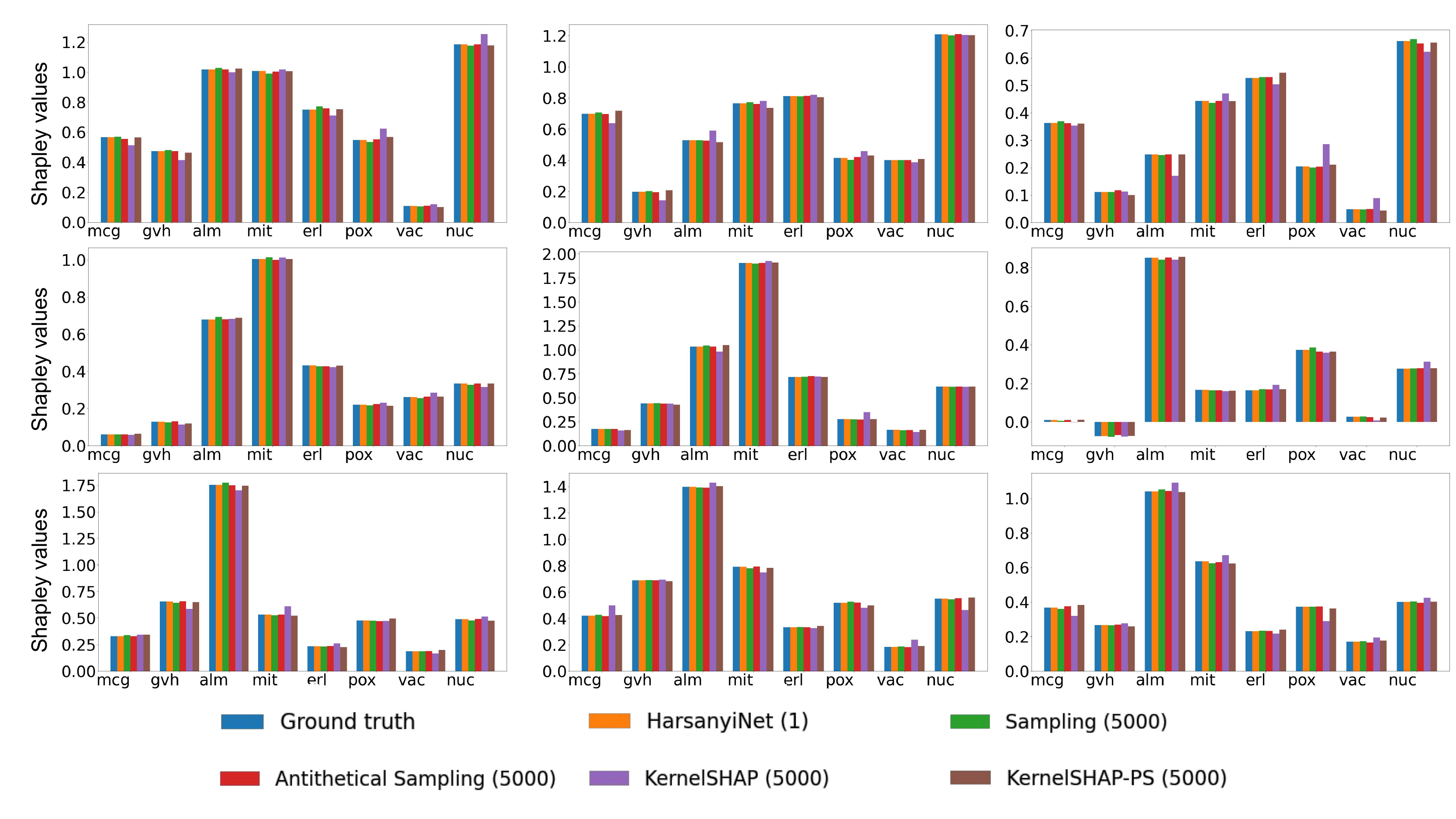}
\caption{Shapley values computed by different methods on the Yeast dataset. The number of inferences conducted for each method is indicated in the brackets. The Shapley values calculated on samples from 3 categories (out of 10 categories) are shown. Samples in the first row are from category `CYT', samples in the second row are from category `MIT', and samples in the last row are from category `NUC'.}
\label{Fig:yeast_example}
\end{figure}

\begin{figure}[t!]
\centering
\includegraphics[width=1.0\linewidth,trim={0cm 2cm 1cm 0cm}]{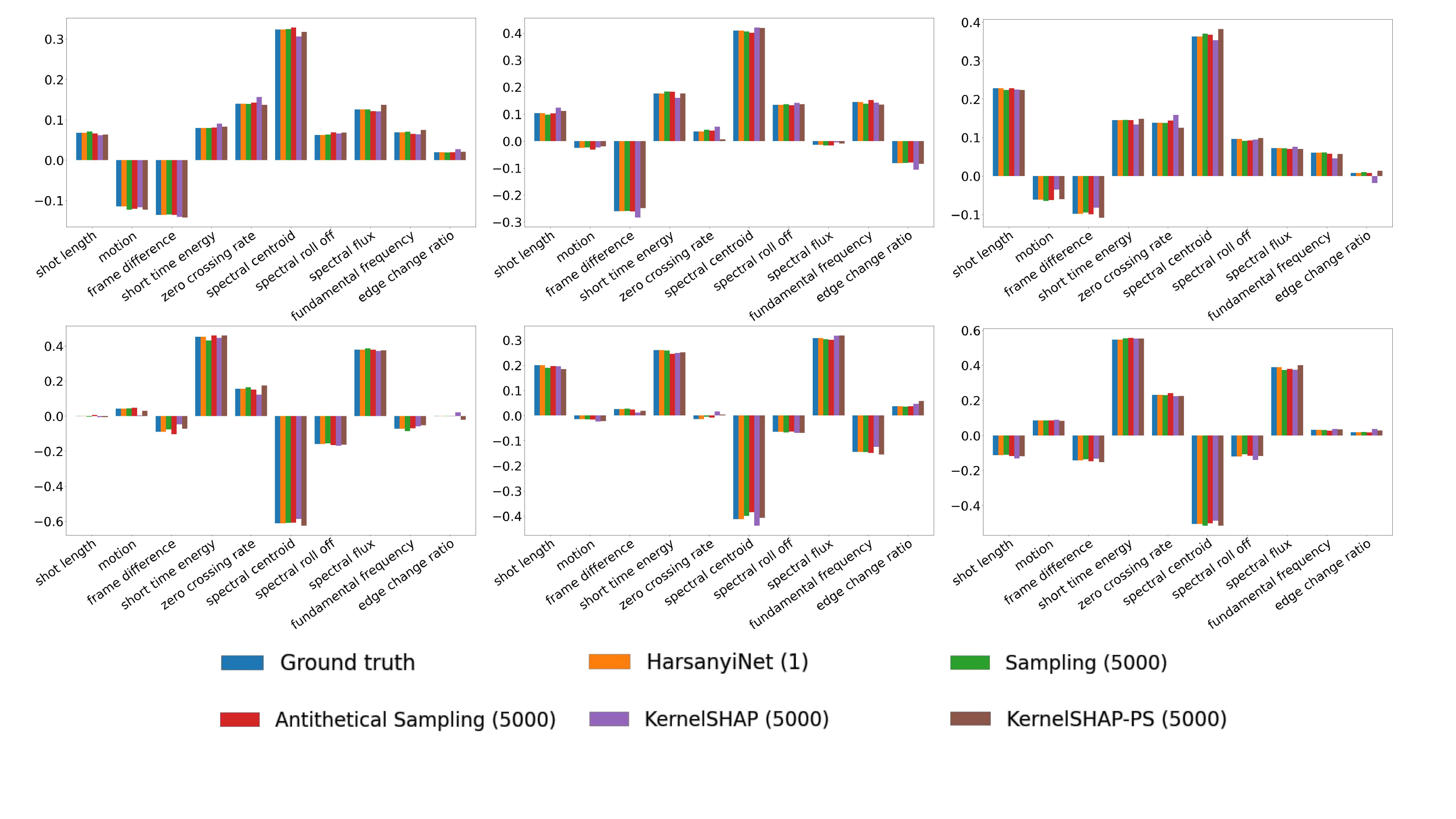}
\caption{Shapley values computed by different methods on the TV News dataset. The number of inferences conducted for each method is indicated in the brackets. The samples in the first row are from category `Non Commercials' and the samples in the second row are from the category `Commercials'.}
\label{Fig:commercial_example}
\end{figure}

\subsection{Experiment details for computing interaction strength of Harsanyi interactions encoded by a DNN}
\label{appendix:aog}
When the number of input variables is small (\eg, $n<16$), we can iteratively calculate the interaction strength of all Harsanyi interactions following Definition~\ref{def:harsanyi}. For tabular datasets, including the Census, Yeast and TV news datasets, we set the baseline value of each input variable the mean value of this variable over all training samples. Then we computed each Harsanyi interaction's strength in a brute-force manner. 

For the MNIST dataset, it is impractical to directly compute the Harsanyi strength of all Harsanyi interactions. To reduce the computational cost, we randomly sampled 8 image regions in the foreground of each image following the previous work~\cite{ren2023defining}. Then we were able to compute the Harsanyi effect of all possible Harsanyi interactions among the sampled 8 image regions (In total we obtained $2^8 = 256$ different Harsanyi interactions). In terms of the baseline value, we set, for each pixel, the baseline value to zero.

\subsection{Experiment details for generating attribution maps in~\cref{Fig:cifar}}
\label{appendix:attrimap}
We compare the attribution maps of each method on the Harsanyi-CNN models. To facilitate comparison with other methods, the Harsanyi-CNN for the MNIST dataset was constructed with 4 cascaded Harsanyi blocks, and each Harsanyi block had $32 \times 14 \times 14$ neurons, where $32$ is the number of channels. The hyperparameters were set to $\beta = 100$ and $\gamma = 0.05$ respectively. The Harsanyi-CNN for the CIFAR-10 dataset was constructed with 10 cascaded Harsanyi blocks, and each Harsanyi block had $256 \times 16 \times 16$ neurons, where $256$ is the number of channels. The hyperparameters were set to $\beta = 1000$ and $\gamma = 1$ respectively. In this way, the HarsanyiNet and the other four model-agnostic methods used the same model to ensure the fairness of the comparison.

Besides, since the Harsanyi-CNN model calculated Shapley values on the feature $\mathbf{z}^{(0)}$, we also calculated Shapley values on $\mathbf{z}^{(0)}$ using the sampling, KernelSHAP, and DeepSHAP methods. For the MNIST dataset, we run about 20000 iterations of the sampling method and 20000 iterations of the KernelSHAP method until convergence. For the CIFAR-10 dataset, we run about 20000 iterations of the sampling method and 200000 iterations of the KernelSHAP method until convergence.

For the FastSHAP method, we used the training samples $\mathbf{x}$ and the model predictions of the Harsanyi-CNN to train a explainer model $\phi_{fast}$, and slightly modified the model architecture to return the attribution maps with a tensor of the same size as the size of $\mathbf{z}^{(0)}$, \ie, $14 \times 14$ for the MNIST dataset and $16 \times 16$ for the CIFAR-10 dataset. For the MNIST dataset, since ~\citep{jethani2021fastshap} did not report the explainer model $\phi_{fast}$, we trained a explainer model with the same structure as which the CIFAR-10 dataset used, and computed the attribution maps on the MNIST dataset.

\subsection{Experiment details for comparing computed Shapley values with true Shapley values on the MNIST and CIFAR-10 datasets}
\label{appendix:experiment12players}
As mentioned in Section~\ref{sec:experimentandcomparison}, in order to reduce the computational cost, we randomly sampled $n=12$ variables in the foreground of the sample as input variables on image datasets. In this way, ground-truth Shapley values were computed by masking the selected 12 variables and keeping all the other variables as the original variables of the sample. Let us denote the set of the selected variables as $\hat{N}$, thus, $|\hat{N}| = 12$. Specifically, we set all the variables $x_i, i\notin \hat{N}$ as the baseline value, \ie, $\forall i\notin \hat{N}, b_i=x_i$ and $\forall i\in \hat{N}, b_i=0$, to obtain a baseline sample $v(\mathbf{x}_{\emptyset})$. Based on the baseline sample $v(\mathbf{x}_{\emptyset})$, we obtained $2^{|\hat{N}|}$ different masked samples. For the HarsanyiNet, when we computed the Shapley values for the selected variables based on Theorem~\ref{theorem1}, we only visited the sets that contain the selected variables, \ie, $S \ni i, \forall i\in \hat{N}$. Besides, $|S|$ denoted the number of the selected variables in $S$. In this way, we computed the Shapley values for the selected variables by $\phi(i) = \sum\nolimits_{S\subseteq N: S \ni i, i\in \hat{N}} \frac{1}{|S|}I(S)$ and $\sum_{i=1}^n \phi(i) = v(\mathbf{x}_N=\mathbf{x}) - v(\mathbf{x}_{\emptyset})$. For the ShapNets, we set all the variables $x_i, i\notin \hat{N}$ as the baseline value $b_i=0$ to obtain a masked sample $\mathbf{x}'$, \ie, $x'_i=x_i, \forall i \in \hat{N}$; $x'_i=0$, otherwise. Then, with the masked input sample, we could compute the Shapley values for the selected variables with the ShapNet. 

\begin{figure}[t!]
\vskip -1.5in
\centering
\includegraphics[width=0.6\linewidth]{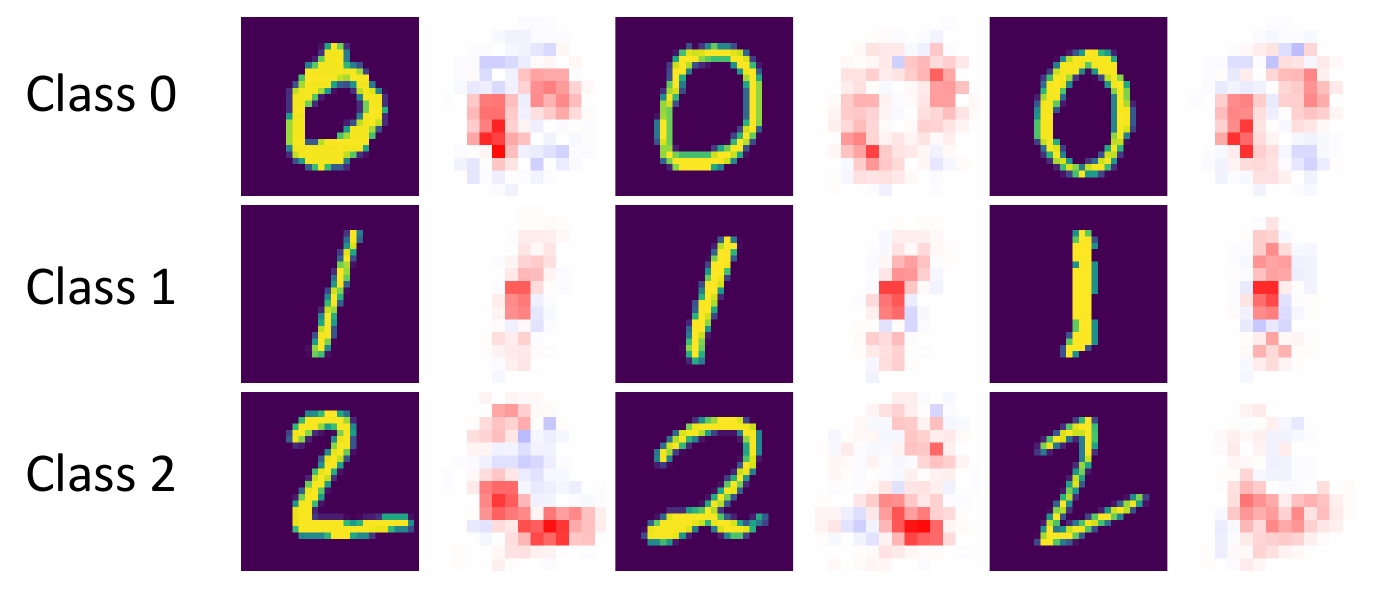}

\caption{Shapley values produced by the HarsanyiNet on the MNIST dataset. The Shapley value is computed by setting $v(\mathbf{x}_S)$ as the output dimension of the ground-truth category of the input sample $\mathbf{x}$.}
\label{Fig:mnist_classes}
\end{figure}

\begin{figure}[t!]
\centering

\includegraphics[width=1.0\linewidth]{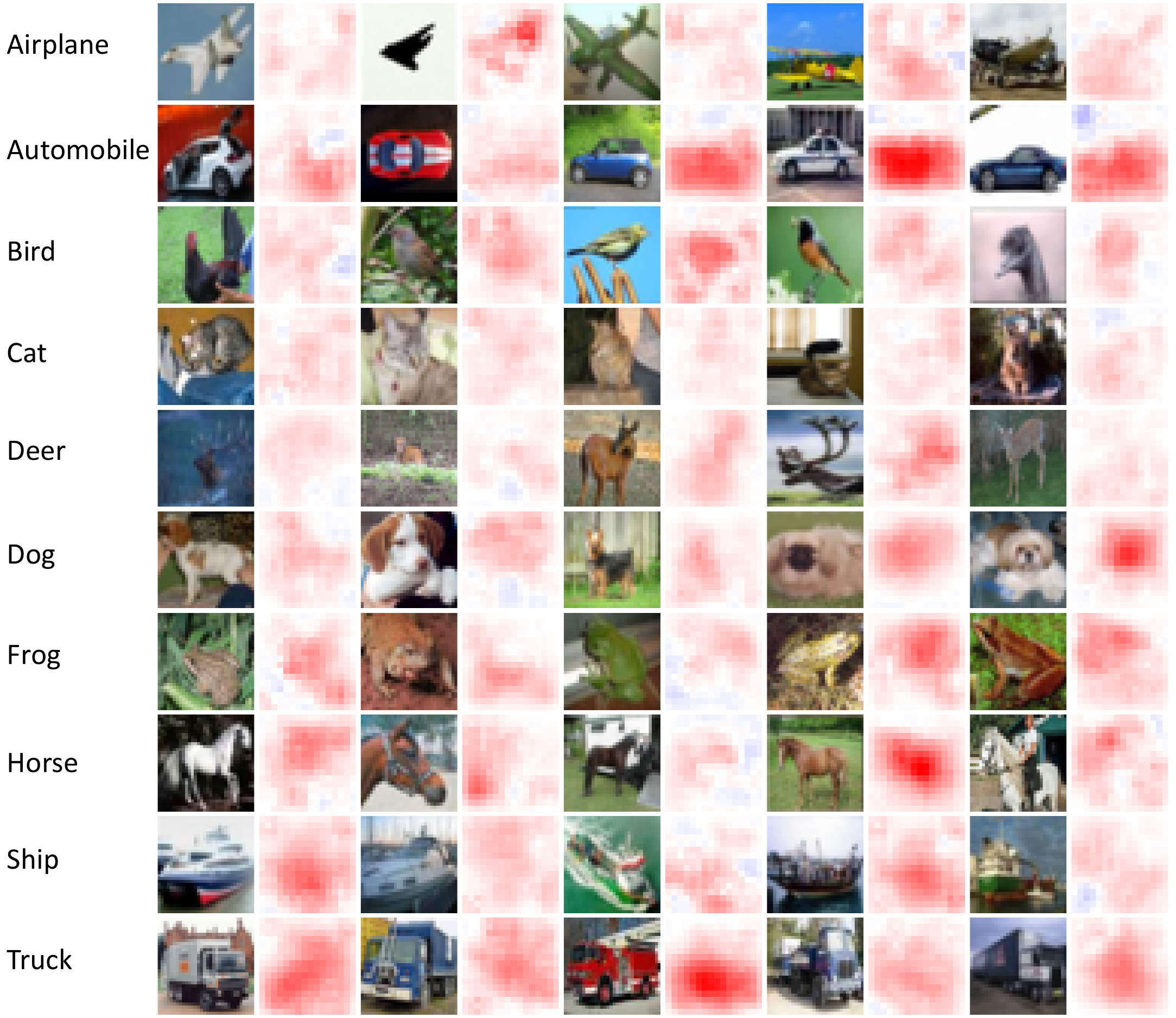}

\caption{Shapley values produced by the HarsanyiNet on the CIFAR-10 dataset. The Shapley value is computed by setting $v(\mathbf{x}_S)$ as the output dimension of the ground-truth category of the input sample $\mathbf{x}$.}
\label{Fig:cifar_classes}
\vskip -0.1in
\end{figure}

\end{document}